%% file: main.tex
\documentclass{colt2018} 

\jmlrvolume{}
\jmlryear{}
\jmlrproceedings{PMLR}{}
\jmlrpages{}

\title[The Externalities of Exploration and How Data Diversity Helps Exploitation]{The Externalities of Exploration and \\ How Data Diversity Helps Exploitation\titletag{\thanks{Extended abstract accepted for presentation at the 31st Conference on Learning Theory (COLT) 2018.}}}

\usepackage{times}

\coltauthor{\Name{Manish Raghavan}\thanks{MR is supported by an NSF Graduate Research Fellowship (DGE-1650441). Work done while at Microsoft Research.} \Email{manish@cs.cornell.edu}\\
     \addr Cornell University, Ithaca, NY, USA.
 \AND
 \Name{Aleksandrs Slivkins} \Email{slivkins@microsoft.com}\\
 \addr Microsoft Research, New York, NY, USA.
 \AND
 \Name{Jennifer Wortman Vaughan} \Email{jenn@microsoft.com}\\
 \addr Microsoft Research, New York, NY, USA.
 \AND
 \Name{Zhiwei Steven Wu} \Email{zsw@umn.edu}\\
 \addr Microsoft Research, New York, NY, USA.
 }

\newcommand{\OMIT}[1]{}
\newcommand{\ie}{{i.e.,~\xspace}}
\newcommand{\eg}{{e.g.,~\xspace}}

\newcommand{\xhdr}[1]{\vspace{2mm} \noindent{\bf #1}}
\newcommand{\polylog}{\operatornamewithlimits{polylog}}
\newcommand{\LDOTS}{\, ,\ \ldots\ ,}     
\newcommand{\Cel}[1]{{\lceil {#1} \rceil}}

\newcommand{\ind}[1]{{\bf 1}_{\{#1\}}} 
\newcommand{\eps}{\varepsilon}

\newcommand{\mE}{\mathcal{E}}
\newcommand{\mH}{\mathcal{H}}

\newcommand{\initOneLiners}{%
 	\setlength{\itemsep}{0pt}
	\setlength{\parsep }{0pt}
  	\setlength{\topsep }{0pt}     	
}
\newenvironment{OneLiners}[1][\ensuremath{\bullet}]
    {\begin{list}
        {#1}
        {\initOneLiners}}
    {\end{list}}

\newenvironment{proofof}[1][]
	{\begin{proof}\textbf{#1}}
	{\end{proof}}

\newcommand{\term}[1]{\ensuremath{\mathtt{#1}}\xspace}

\newcommand{\ALG}{\term{ALG}}

\newcommand{\simF}{\term{sim}} 

\newcommand{\con}{\term{con}} 
\newcommand{\CON}{\term{CON}} 

\newcommand{\PReg}{\term{PReg}} 
\newcommand{\rPReg}[1][{\rounds}]{\PReg^{#1}}  
\newcommand{\rReg}[1][{\rounds}]{R^{#1}} 

\newcommand{\bpreg}[1]{\PReg^{\rounds}(#1)}
\newcommand{\bpregi}[1]{\PReg^{#1}}

\newcommand{\basereg}[1]{R_0^{\rounds}\left(#1\right)}

\newcommand{\regi}[1]{R^{#1}}

\newcommand{\BayesGreedy}{BatchBayesGreedy\xspace}
\newcommand{\FreqGreedy}{BatchFreqGreedy\xspace}
\newcommand{\GreedyStyle}{batch-greedy-style\xspace}

\newcommand{\Rew}[1][\theta]{\mathtt{Rew}_{#1}} 

\def\rounds{\mc T}

\usepackage{cleveref}
\usepackage{color}
\usepackage{tikz}
\newtheorem{fact}[theorem]{Fact}

\def\R{\mathbb{R}}
\def\tran{^\top}

\newcommand{\p}[1]{\left(#1\right)}
\newcommand{\kl}[2]{\ensuremath{\text{KL}(#1 \, ||\, #2)}}
 \def\given{\;|\;}

\newcommand{\thetazero}{\theta^{(0)}}
\newcommand{\thetaone}{\theta^{(1)}}

\renewcommand{\b}[1]{\left[#1\right]}
\newcommand{\E}[1]{\mathbb{E}\b{#1}}
\newcommand{\D}{D}

\DeclareMathOperator*{\argmax}{arg\,max}
\DeclareMathOperator*{\argmin}{arg\,min}
\DeclareMathOperator{\Exp}{\mathbb{E}}
\def\R{\mathbb{R}}
\def\ty{{\lfloor t/Y \rfloor}}

\newcommand{\mc}[1]{\mathcal{#1}}
\def\tran{^{\top}}
\def\dx{\;dx}
\def\given{\;|\;}

\def\pmt{\overline \theta}
\def\pvt{\Sigma}
\def\bmt{\theta_t^{\textrm{bay}}}

\def\ab{a_t'}

\newcommand{\fmt}[1][t]{\theta_{#1}^{\textrm{fre}}}

\def\af{a_t}
\def\vrt{r_B}
\def\vrto{\mathbf{r}_{1:t_0-1}}
\def\Xto{X_{t_0-1}}
\def\Zto{Z_{t_0-1}}
\def\weights{w_B}
\def\prior{\mc P}
\def\thetahatt{\hat{\theta}_t}
\def\thetahatti{(\thetahatt)_i}

\newcommand{\creg}[2]{\text{Regret}^{#1}(#2)}

\newcommand{\iR}[1]{R_{#1}} 

\def\hcat{\hat{c}_{a,t}}

\def\ZB{Z_B}
\def\WB{W_B}

\def\bg{\BayesGreedy}

\def\fg{\FreqGreedy}

\def\elt{\ensuremath{E_{\ell, t}}}

\begin{document}

\maketitle

\begin{abstract}
  \input{sections/abstract}
\end{abstract}

\newpage
\setcounter{tocdepth}{2} 
\tableofcontents
\vfill
\newpage

\section{Introduction}
\label{sec:intro}
\input{sections/intro}

\section{Preliminaries}
\label{sec:model}
\input{sections/model}

\newpage
\section{Group Externality of Exploration}
\label{sec:negative}
\input{sections/negative}

\section{Greedy Algorithms and LinUCB with Perturbed Contexts}
\label{sec:bayesian_greedy}
\input{sections/bayesian_greedy}

\section{Analysis: LinUCB with Perturbed Contexts}
\label{app:linucb}
\input{sections/linucb}

\section{Analysis: Greedy Algorithms with Perturbed Contexts}
\label{app:pf_bg}
\input{sections/bg_proofs}

\acks{We thank Solon Barocas, Dylan Foster, Jon Kleinberg, Aaron Roth, and Hanna Wallach for helpful discussions about these topics.}

\bibliography{refs}

\appendix

\section{Auxiliary Lemmas}
\label{app:lemmas}
\input{sections/lemmas}
\end{document}

%% file: sections/abstract.tex
Online learning algorithms, widely used to power search and content optimization on the web, must balance exploration and exploitation, potentially sacrificing the experience of current users in order to gain information that will lead to better decisions in the future.  Recently, concerns have been raised about whether the process of exploration could be viewed as unfair, placing too much burden on certain individuals or groups.  Motivated by these concerns, we initiate the study of the externalities of exploration---the undesirable side effects that the presence of one party may impose on another---under the linear contextual bandits model.  We introduce the notion of a group externality, measuring the extent to which the presence of one population of users (the majority) impacts the rewards of another (the minority). We show that this impact can, in some cases, be negative, and that, in a certain sense, no algorithm can avoid it.  We then move on to study externalities at the individual level, interpreting the act of exploration as an externality imposed on the current user of a system by future users. This drives us to ask under what conditions inherent diversity in the data makes explicit exploration unnecessary.  We build on a recent line of work on the smoothed analysis of the greedy algorithm that always chooses the action that currently looks optimal. We improve on prior results to show that a greedy approach almost matches the best possible Bayesian regret rate of any other algorithm on the same problem instance whenever the diversity conditions hold, and that this regret is at most $\tilde{O}(T^{1/3})$. Returning to group-level effects, we show that under the same conditions, negative group externalities essentially vanish if one runs the greedy algorithm. Together, our results uncover a sharp contrast between the high externalities that exist in the worst case, and the ability to remove all externalities if the data is sufficiently diverse. 

%% file: sections/intro.tex
Online learning algorithms are a key tool in web search and content optimization, adaptively learning what users want to see. In a typical application, each time a user arrives, the algorithm chooses among various content presentation options (\eg news articles to display), the chosen content is presented to the user, and an outcome (\eg a click) is observed. Such algorithms must balance \emph{exploration} (making potentially suboptimal decisions now for the sake of acquiring information that will improve decisions in the future) and \emph{exploitation} (using information collected in the past to make better decisions now). Exploration could degrade the experience of a current user, but improves user experience in the long run. This exploration-exploitation tradeoff is commonly studied in the online learning framework of \emph{multi-armed bandits}~\citep{Bubeck-survey12}.

Concerns have been raised about whether exploration in such scenarios could be unfair, in the sense that some individuals or groups may experience too much of the downside of exploration without sufficient upside \citep{bird2016exploring}. We formally study these concerns in the \emph{linear contextual bandits} model~\citep{Langford-www10,chu2011contextual}, a standard variant of multi-armed bandits appropriate for content personalization scenarios.  We focus on \emph{externalities} arising due to exploration, that is, undesirable side effects that the presence of one party may impose on another.

We first examine the effects of exploration at a group level.  We introduce the notion of a \emph{group externality} in an online learning system, quantifying how much the presence of one population (which we dub the majority) impacts the rewards of another (the minority). We show that this impact can be negative, and that, in a particular precise sense, no algorithm can avoid it. This cannot be explained by the absence of suitably good policies since our adoption of the linear contextual bandits framework implies the existence of a feasible policy that is simultaneously optimal for everyone. Instead, the problem is inherent to the process of exploration. We come to a surprising conclusion that more data can sometimes lead to worse outcomes for the users of an explore-exploit-based system. \looseness=-1

We next turn to the effect of exploration at an individual level. We interpret exploration as a potential externality imposed on the current user by future users of the system. Indeed, it is only for the sake of the future users that the algorithm would forego the action that currently looks optimal. To avoid this externality, one may use the greedy algorithm that always chooses the action that appears optimal according to current estimates of the problem parameters. While this greedy algorithm performs poorly in the worst case,
it tends to work well in many applications and experiments.\footnote{Both positive and negative findings are folklore. One way to precisely state the negative result is that the greedy algorithm incurs constant per-round regret with constant probability; while results of this form have likely been known for decades,
\citet[Corollary A.2(b)]{competingBandits-itcs16}
proved this for a wide variety of scenarios. Very recently, the good empirical performance has been confirmed by state-of-art experiments in \citet{practicalCB-arxiv18}.}

In a new line of work, \citet{bastani2017exploiting} and \citet{kannan2018smoothed}
analyzed conditions under which inherent diversity in the data makes explicit exploration unnecessary.
\citet{kannan2018smoothed} proved that the greedy algorithm achieves a regret rate of
$\tilde{O}(\sqrt{T})$ in expectation over small perturbations of the context vectors (which ensure sufficient data diversity). This is the best rate that can be achieved in the worst case (\ie for all problem instances, without data diversity assumptions), but it leaves open the possibilities that (i) another algorithm may perform much better than the greedy algorithm on some problem instances, or (ii) the greedy algorithm may perform much better than worst case under the diversity conditions. We expand on this line of work. We prove that under the same diversity conditions, the greedy algorithm almost matches the best possible Bayesian regret rate of \emph{any} algorithm \emph{on the same problem instance}. This could be as low as $\polylog(T)$ for some instances, and, as we prove, at most $\tilde{O}(T^{1/3})$ whenever the diversity conditions hold.

Returning to group-level effects, we show that under the same diversity conditions, the negative group externalities imposed by the majority essentially vanish if one runs the greedy algorithm. Together, our results illustrate a sharp contrast between the high individual and group externalities that exist in the worst case, and the ability to remove all externalities if the data is sufficiently diverse.   \looseness=-1

\xhdr{Additional motivation.} Whether and when explicit exploration is necessary is an important concern in the study of the exploration-exploitation tradeoff. Fairness considerations aside, explicit exploration is expensive. It is wasteful and risky in the short term, it adds a layer of complexity to algorithm design \citep{Langford-nips07,monster-icml14}, and its adoption at scale tends to require substantial systems support and buy-in from management \citep{MWT-WhitePaper-2016,DS-arxiv}. A system based on the greedy algorithm would typically be cheaper to design and deploy.

Further, explicit exploration can run into incentive issues in applications such as recommender systems. Essentially, when it is up to the users which products or experiences to choose and the algorithm can only issue recommendations and ratings, an explore-exploit algorithm needs to provide incentives to explore for the sake of the future users \citep{Kremer-JPE14,Frazier-ec14,Che-13,ICexploration-ec15,Bimpikis-exploration-ms17}. Such incentive guarantees tend to come at the cost of decreased performance, and rely on assumptions about human behavior. The greedy algorithm avoids this problem as it is inherently consistent with the users' incentives.

\xhdr{Additional related work.}
Our research draws inspiration from the growing body of work on fairness in machine learning~\cite[e.g.,][]{dwork2012fairness,hardt2016equality,kleinberg2017inherent,chouldechova2017fair}.  Several other authors have studied fairness in the context of the contextual bandits framework.  Our work differs from the line of research on meritocratic fairness in online learning \citep{kearns2017meritocratic,liu2017calibrated,joseph2016fairness}, which considers the allocation of limited resources such as bank loans and requires that nobody should be passed over in favor of a less qualified applicant. We study a fundamentally different scenario in which there are no allocation constraints and we would like to serve each user the best content possible.  Our work also differs from that of \citet{celis2017fair}, who studied an alternative notion of fairness in the context of news recommendations. According to this notion, all users should have approximately the same probability of seeing a particular type of content (e.g., Republican-leaning articles), regardless of their individual preferences, in order to mitigate the possibility of discriminatory personalization.

The data diversity conditions in \citet{kannan2018smoothed} and this paper are inspired by the smoothed analysis framework of \citet{SmoothedAnalysis-jacm04}, who proved that the expected running time of the simplex algorithm is polynomial for perturbations of any initial problem instance (whereas the worst-case running time has long been known to be exponential). Such disparity implies that very bad problem instances are brittle. 
We find a similar disparity for the greedy algorithm in our setting.

\xhdr{Our results on group externalities.}  A typical goal in online learning is to minimize \emph{regret}, the (expected) difference between the cumulative reward that would have been obtained had the optimal policy been followed at every round and the cumulative reward obtained by the algorithm.  We define a corresponding notion of \emph{minority regret}, the portion of the regret experienced by the minority.  Since online learning algorithms update their behavior based on the history of their observations, minority regret is influenced by the entire population on which an algorithm is run.  If the minority regret is much higher when a particular algorithm is run on the full population than it is when the same algorithm is run on the minority alone, we can view the majority as imposing a negative externality on the minority; the minority population would achieve a higher cumulative reward if the majority were not present. Asking whether this can ever happen
amounts to asking whether access to more data points can ever lead an explore-exploit algorithm to make inferior decisions. One might think that more data should always lead to better decisions and therefore better outcomes for the users.
Surprisingly, we show that this is not the case, even with a standard algorithm.

Consider LinUCB~\citep{Langford-www10,chu2011contextual,abbasi2011improved}, a standard algorithm for linear contextual bandits that is based on the principle of ``optimism under uncertainty.''  We provide a specific problem instance on which, after observing $T$ users, LinUCB would have a minority regret of $\Omega(\sqrt T)$ if run on the full population, but only constant minority regret if run on the minority alone. While stylized, this example is motivated by the problem of providing driving directions to different populations of users, about which fairness concerns have been raised~\citep{bird2016exploring}. Further, the situation is reversed on a slight variation of this example: LinUCB obtains constant minority regret when run on the full population and $\Omega(\sqrt T)$ on the minority alone.  That is, group externalities can be large and positive in some cases, and large and negative in others.

Although these regret rates are specific to LinUCB, we show that this phenomenon is, in some sense, unavoidable. Consider the minority regret of LinUCB when run on the full population and the minority regret that LinUCB would incur if run on the minority alone. We know that one may be much smaller or larger than the other. We ask whether any algorithm could  achieve the minimum of the two on every problem instance. Using a variation of the same problem instance, we prove that this is impossible; in fact, no algorithm could simultaneously approximate both up to any $o(\sqrt{T})$ factor. In other words, an externality-free algorithm would sometimes ``leave money on the table."

In terms of techniques, we rely on the special structure of our example, which can be viewed as an instance of the sleeping bandits problem~\citep{SleepingBandits-ml10}. This simplifies the behavior and analysis of LinUCB, allowing us to obtain the $O(1)$ upper bounds.  The lower bounds are obtained using KL-divergence techniques to show that the two variants of our example are essentially indistinguishable, and an algorithm that performs well on one must obtain $\Omega(\sqrt{T})$ regret on the other. \looseness=-1

\xhdr{Our results on the greedy algorithm.} We consider a version of linear contextual bandits in which the latent weight vector $\theta$ is drawn from a known prior. In each round, an algorithm is presented several actions to choose from, each represented by a \emph{context vector}. The expected reward of an action is a linear product of $\theta$ and the corresponding context vector. The tuple of context vectors is drawn independently from a fixed distribution. In the spirit of smoothed analysis, we assume that this distribution has a small amount of jitter. Formally, the tuple of context vectors is drawn from some fixed distribution, and then a small \emph{perturbation}---small-variance Gaussian noise---is added independently to each coordinate of each context vector. This ensures arriving contexts are diverse. We are interested in Bayesian regret, i.e., regret in expectation over the Bayesian prior. Following the literature, we are primarily interested in the dependence on the time horizon $T$. \looseness=-1

We focus on a batched version of the greedy algorithm, in which new data arrives to the algorithm's optimization routine in small batches, rather than every round. This is well-motivated from a practical perspective---in high-volume applications data usually arrives to the ``learner" only after a substantial delay \citep{MWT-WhitePaper-2016,DS-arxiv}---and is essential for our analysis.

Our main result is that the greedy algorithm matches the Bayesian regret of any algorithm up to polylogarithmic factors, for each problem instance, fixing the Bayesian prior and the context distribution. We also prove that LinUCB achieves regret $\tilde{O}(T^{1/3})$ for each realization of $\theta$. This implies a worst-case Bayesian regret of $\tilde{O}(T^{1/3})$ for the greedy algorithm under the perturbation assumption. \looseness=-1

Our results hold for both natural versions of the batched greedy algorithm, Bayesian and frequentist, henceforth called \BayesGreedy and \FreqGreedy. In \BayesGreedy, the chosen action maximizes expected reward according to the Bayesian posterior. \FreqGreedy estimates $\theta$ using ordinary least squares regression and chooses the best action according to this estimate. The results for \FreqGreedy come with additive polylogarithmic factors, but are stronger in that the algorithm does not need to know the prior. Further, the $\tilde{O}(T^{1/3})$ regret bound for \FreqGreedy is approximately prior-independent, in the sense that it applies even to very concentrated priors such as independent Gaussians with standard deviation on the order of $T^{-2/3}$.

The key insight in our analysis of \BayesGreedy is that any (perturbed) data can be used to simulate any other data, with some discount factor. The analysis of \FreqGreedy requires an additional layer of complexity. We consider a hypothetical algorithm that receives the same data as \FreqGreedy, but chooses actions based on the Bayesian-greedy selection rule. We analyze this hypothetical algorithm using the same technique as \BayesGreedy, and then upper bound the difference in Bayesian regret between the hypothetical algorithm and \FreqGreedy.

Our analyses extend to group externalities and (Bayesian) minority regret. In particular, we circumvent the impossibility result mentioned above. We prove that both \BayesGreedy and \FreqGreedy match the Bayesian minority regret of any algorithm run on either the full population or the minority alone, up to polylogarithmic factors

\xhdr{Detailed comparison with prior work.} We substantially improve over the $\tilde{O}(\sqrt{T})$ worst-case regret bound from \citet{kannan2018smoothed}, at the cost of some additional assumptions. First, we consider Bayesian regret, whereas their regret bound is for each realization of $\theta$.%
\footnote{Equivalently, they allow point priors, whereas our priors must have variance $T^{-O(1)}$.} Second, they allow the context vectors to be chosen by an adversary before the perturbation is applied. Third, they extend their analysis to a somewhat more general model, in which there is a separate latent weight vector for every action (which amounts to a more restrictive model of perturbations). However, this extension relies on the greedy algorithm being initialized with a substantial amount of data. The results of \citet{kannan2018smoothed} do not appear to have implications on group externalities.

\citet{bastani2017exploiting} show that the greedy algorithm achieves logarithmic regret in an alternative linear contextual bandits setting that is incomparable to ours in several important ways.
They consider two-action instances where the actions share a common context vector in each round, but are parameterized by different latent vectors. They ensure data diversity via a strong assumption on the context distribution. This assumption does not follow from our perturbation conditions; among other things, it implies that each action is the best action in a constant fraction of rounds. Further, they assume a version of Tsybakov's \emph{margin condition}, which is known to substantially reduce regret rates in bandit problems \citep[\eg see][]{Zeevi-colt10}.

%% file: sections/model.tex
We consider the model of \emph{linear contextual
  bandits}~\citep{Langford-www10,chu2011contextual}. Formally, there
is a learner who serves a sequence of users over $T$ rounds, where $T$
is the (known) time horizon.  For the user who arrives in round $t$ there are
(at most) $K$ actions available, with each action
$a\in \{1, \ldots , K\}$ associated with a \emph{context vector}
$x_{a,t} \in \R^d$. Each context vector may contain a mix of features
of the action, features of the user, and features of both.  We assume
that the tuple of context vectors for each round $t$ is drawn
independently from a fixed distribution.  The learner observes the set
of contexts and selects an action $a_t$ for the user. The user then
experiences a reward $r_t$ which is visible to the learner. We assume
that the expected reward is linear in the chosen context vector. More
precisely, we let $r_{a,t}$ be the reward of action $a$ if this action
is chosen in round $t$ (so that $r_t = r_{a_t,t}$), and assume that
there exists an unknown vector $\theta\in\R^d$ such that
$\E{r_{a, t}\mid x_{a, t}} = \theta\tran x_{a, t}$ for any round $t$
and action $a$. Throughout most of the paper, the realized rewards are
either in $\{0,1\}$ or are the expectation plus independent Gaussian
noise of variance at most $1$. We sometimes consider a Bayesian
version, in which the latent vector $\theta$ is initially drawn from some
known prior $\prior$. \looseness=-1

A standard goal for the learner is to maximize the expected total
reward over $T$ rounds, or $\sum_{t=1}^T \theta\tran x_{a, t}$. This
is equivalent to minimizing the learner's \emph{regret}, defined as 
\begin{align}\label{eq:regret-def}
\text{Regret}(T) = \textstyle
    \sum_{t=1}^T \theta\tran x_t^* -
\theta\tran x_{a_t, t} 
\end{align}
 where $x^*_{t} = \argmax_{x \in
\{x_{1,t}, \cdots, x_{K,t}\}} \theta\tran x$ denotes the
context vector associated with the best action at round $t$. We are mainly
interested in \emph{expected regret}, where the expectation is taken over the
context vectors, the rewards, and the algorithm's random seed, and
\emph{Bayesian regret}, where the expectation is taken over all of the above and
the prior over $\theta$.

We introduce some notation in order to describe and
analyze algorithms in this model. We write $x_t$ for $x_{a_t,t}$, the context vector chosen at time $t$. Let $X_t\in \R^{t\times d}$ be the \emph{context matrix} at time
$t$, a matrix whose rows are vectors $x_1 \LDOTS x_t\in \R^d $.
A $d\times d$ matrix
    $Z_t :=\sum_{\tau=1}^t x_\tau x_\tau\tran = X_t\tran X_t$,
called the \emph{empirical covariance} matrix, is an important concept in some of the prior work on linear contextual bandits
\citep[\eg][]{abbasi2011improved,kannan2018smoothed}, as well as in this paper.

\xhdr{Optimism under uncertainty.}  Optimism under uncertainty is a
common paradigm in problems with an explore-exploit
tradeoff~\citep{Bubeck-survey12}. The idea is to evaluate each action
``optimistically"---assuming the best-case scenario for this
action---and then choose an action with the best optimistic
evaluation. When applied to the basic multi-armed bandit setting, it
leads to a well-known algorithm called UCB1 \citep{bandits-ucb1},
which chooses the action with the highest upper confidence bound (henceforth,
UCB) on its mean reward. The UCB is computed as the sample average of
the reward for this action plus a term which captures the amount of
uncertainty.

Optimism under uncertainty has been extended to linear contextual
bandits in the LinUCB
algorithm~\citep{chu2011contextual,abbasi2011improved}. The high-level
idea is to compute a confidence region
$\Theta_t \subset \R^d$ in each round $t$ such that $\theta\in \Theta_t$ with high
probability, and choose an action $a$ which maximizes the optimistic
reward estimate $\sup_{\theta \in \Theta_t} x_{a,t}\tran \theta$.
Concretely, one uses regression to form an empirical estimate
$\thetahatt$ for $\theta$.  Concentration techniques lead to
high-probability bounds of the form
$|x\tran (\theta -\thetahatt)| \leq f(t) \sqrt{x\tran Z_t^{-1}x}$, where
the \emph{interval width function} $f(t)$ may depend on hyperparameters and features of
the instance. LinUCB simply chooses an action
\begin{equation}
  a_t^{LinUCB} := \argmax_a x_{a,t}\tran \thetahatt + f(t) \sqrt{x_{a,t}\tran Z_t^{-1}
  x_{a,t}}.
  \label{eq:linucb_def}
\end{equation}
Among other results, \citet{abbasi2011improved} use
\begin{equation}
  f(t) = S+\sqrt{d c_0\log (T+tTL^2)},
  \label{eq:Abbasi-f}
\end{equation}
where $L$ and $S$ are known upper bounds on $\|x_{a,t}\|_2$ and $\|\theta\|_2$, respectively, and $c_0$ is a parameter. For any $c_0\geq 1$, they obtain regret
    $\tilde{O}(dS\sqrt{c_0\, K\,T})$,
with only a $\polylog$ dependence on $TL/d$.

%% file: sections/negative.tex
In this section, we study the externalities of exploration at a group
level, quantifying how much the presence of one population impacts the
rewards of another in an online learning system. We consider linear
contextual bandits in a setting in which there are two underlying user
populations, called the \emph{majority} and the \emph{minority}. The
user who arrives at round $t$ is assumed to come from the majority
population with some fixed probability and the minority population
otherwise, and the population from which the user comes is known to
the learner. The tuple of context vectors at time $t$ is then drawn
independently from a fixed group-specific distribution.

We assume there is a
single hidden vector $\theta$, and that the distribution of rewards conditioned
on the chosen context vector is
the same for both groups. Only the distribution over tuples of
available context vectors differs between groups. This implies that
externalities cannot be explained by the absence of a good policy,
since there always exists a policy that is simultaneously optimal for
everyone. This allows us to focus only on externalities inherent to the
process of exploration.

We define the \emph{minority regret} to be the regret experienced by
the minority. The \emph{group
  externality} imposed on the minority by the majority is then the
difference between the minority regret of an algorithm run on the
minority alone and the minority regret of the same algorithm run on
the full population. A negative group externality implies that the
minority is worse off due to the presence of the majority.
It is generally more meaningful to bound the multiplicative
difference between the minority regret obtained with and without the
majority present. Several of our results have this form. \looseness=-1

We first ask whether large group externalities can exist. We show that
on a simple toy example, a large negative group externality arises
under LinUCB, while a slight variant of this example leads to a large
positive externality. Put another way, more available data can lead to
either better or worse outcomes for the users of a system. We show
that this general phenomenon is unavoidable. That is, no algorithm can
simultaneously approximate the minority regret of LinUCB run on the
full population and LinUCB run on the minority alone, up to any
$o(\sqrt{T})$ multiplicative factor.

\subsection{Two-Bridge Instance}
We consider a toy example, motivated by a scenario in which a learner
is choosing driving routes for two groups of users. Each user starts
at point $A$, $B$, or $C$, and wants to get to the same destination,
point $D$, which requires taking one of two bridges, as shown in
Figure~\ref{fig:bridges}. The travel costs for each of the two bridges
are unknown. For simplicity, assume all other edges are known to have
0 cost.

Suppose that 95\% of users are in the majority group. All of these
users start at point $A$ and have access only to the top bridge. The
other 5\% are in the minority. Of these users, 95\% start at point
$C$, from which they have access only to the bottom bridge. The
remaining 5\% of the minority users start at point $B$, and have
access to both bridges.

\begin{figure}[h]
\centering
\begin{tikzpicture}[scale=.75]
\node[circle, draw] (A) at (-2,-1) {$A$};
\node[circle, draw] (B) at (-1,-2) {$B$};
\node[circle, draw] (C) at (-2,-3) {$C$};
\node[circle, draw] (s1) at (1,-1) {};
\node[circle, draw] (s2) at (1,-3) {};
\node[circle, draw] (e1) at (4,-1) {};
\node[circle, draw] (e2) at (4,-3) {};
\node[circle, draw] (D) at (6,-2) {$D$};
\draw[->,draw,>=stealth] (A) -- (s1);
\draw[->,draw,>=stealth] (B) -- (s1);
\draw[->,draw,>=stealth] (B) -- (s2);
\draw[->,draw,>=stealth] (C) -- (s2);
\draw[->,draw,>=stealth,very thick] (s1) -- (e1) node[midway,above] {$\theta_1$};
\draw[->,draw,>=stealth,very thick] (s2) -- (e2) node[midway,above] {$\theta_2$};
\draw[->,draw,>=stealth] (e1) -- (D);
\draw[->,draw,>=stealth] (e2) -- (D);
\end{tikzpicture}
\caption{Visual illustration of the two-bridge instance. \label{fig:bridges}}
\end{figure}
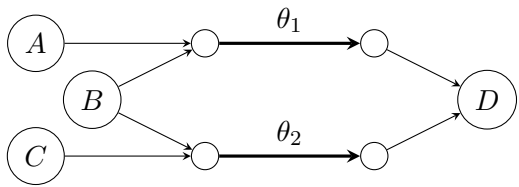

Consider the behavior of an algorithm that follows the principle of
optimism under uncertainty.  If run on the full user population, it
will quickly collect many observations of the commute time for the top
bridge since all users in the majority group must travel over the top
bridge.  It will collect relatively fewer observations of the commute
time over the bottom bridge.  Therefore, when the algorithm is faced
with a member of the minority population who starts at point $B$, the
algorithm will likely send this user over the bottom bridge in order
to collect more data and improve its estimate. \looseness=-1

If the same algorithm is instead run on the minority alone, it will
quickly collect many more observations of the commute time for the
bottom bridge relative to the top.
Now when the algorithm is faced with a user who starts
at point $B$, it will likely send her over the top bridge.

Which is better depends on which bridge has the longer commute time.
If the top bridge is the better option, then the presence of the
majority imposes a negative externality on the minority. If not, then
the presence of the majority helps. These two scenarios may be
difficult to distinguish.

This toy example can be formalized in the linear contextual bandits
framework. There are two underlying actions (the two bridges), but
these actions are not always available. To capture this, we define a
parameter vector $\theta$ in $[0, 1]^2$, with the two coordinates
$\theta_1$ and $\theta_2$ representing the expected rewards for taking
the top and bottom bridge respectively. (Though we motivated
the example in terms of costs, it can be expressed equivalently in
terms of rewards.) There are two possible context vectors:
$[1 ~ 0]\tran$ and $[0 ~ 1]\tran$. A user has available an action with
context vector $[1 ~ 0]\tran$ if and only if she has access to the top
bridge. Similarly, she has available an action with context vector
$[0 ~ 1]\tran$ if and only if she has access to the bottom bridge. The
instance can then be formalized as follows.

\begin{definition}[Two-Bridge Instance]
  The \emph{two-bridge instance} is an instance of linear contextual
  bandits. On each round $t$, the user who arrives is from the
  majority population with probability 0.95, in which case
  $x_{1, t} = x_{2, t} = [1 ~ 0]\tran$. Otherwise, the user is in the
  minority. In this case, with probability 0.95,
  $x_{1,t} = x_{2,t} = [0 ~ 1]\tran$ (based on
  Figure~\ref{fig:bridges}, we call these $C$ rounds), while with probability 0.05, $x_{1,t} = [1 ~ 0]\tran$ and
  $x_{2,t} = [0 ~ 1]\tran$ ($B$ rounds). We consider two
  values for the hidden parameter vector $\theta$,
  $\thetazero = [1/2 \ ~ 1/2 - \varepsilon]\tran$ and
  $\thetaone = [1/2 - \varepsilon \ ~ 1/2]\tran$ where
  $\varepsilon = 1/\sqrt{T}$.
\end{definition}

\subsection{Performance of LinUCB}

We start by analyzing the performance of LinUCB on the two-bridge
instance.  Our main result formalizes the intuition above, showing
that when $\theta = \thetazero$ (that is, the top bridge is better)
the majority imposes a large negative group externality on the
minority, while the majority imposes a large positive externality when
$\theta = \thetaone$. We assume rewards are $1$-subgaussian.%
\footnote{A random variable $X$ is called $\sigma$-subgaussian if $E[e^{\sigma X^2}]<\infty$. A special case is Gaussians with variance $\sigma^2$.}

\begin{theorem}
  Consider LinUCB with any interval width function $f$ satisfying
  $f(t) \ge 2\sqrt{\log(T)}$.%
  \footnote{For instance, the interval width function in Equation~\eqref{eq:Abbasi-f} satisfies this condition whenever $d c_0\geq 4$, so one can either set $c_0\geq 2$, or add two more dimensions to the problem instance (and set $\theta_3 = \theta_4 =0$).}
   On the two-bridge instance,
  assuming $1$-subgaussian noise on the rewards, when
  $\theta = \thetazero$, LinUCB achieves expected minority regret
  $O(1)$
  when run on the minority alone, but $\Omega(\sqrt{T})$ when run on
  the full population.  In contrast, when $\theta = \thetaone$, LinUCB
  achieves expected minority regret $O(1)$ when run on the full
  population, but $\Omega(\sqrt{T})$ when run on the minority alone.
\label{thm:twobridgelinucb}
\end{theorem}

We omit the proofs of the $\Omega(\sqrt{T})$ lower bounds, which both
follow a similar structure to the one used in the proof of the general
impossibility result in Section~\ref{sec:impossibility}; in fact, both
of these lower bounds could be stated as an immediate corollary of
Theorem~\ref{thm:indistinguishability}.  Essentially, an argument
based on KL-divergence shows that it is difficult to distinguish
between the case in which $\theta = \thetazero$ and the case in which
$\theta = \thetaone$, and therefore LinUCB must choose similar actions
in these two cases.

To prove the $O(1)$ upper bounds, we make heavy use of the special
structure of the two-bridge instance, which significantly simplifies
the analysis of LinUCB. We exploit the fact that the only context
vectors available to the learner are the basis vectors $[1 ~ 0]\tran$
and $[0 ~ 1]\tran$, which essentially makes this an instance of
sleeping bandits~\citep{SleepingBandits-ml10}. In this special case,
the covariance matrix $Z_t$ is always diagonal, which simplifies
Equation~\eqref{eq:linucb_def} and leads to LinUCB choosing the $i$th
basis, where $i$ maximizes $(\thetahatt)_i + f(t)/\sqrt{(Z_t)_{ii}}$ and
$(Z_t)_{ii}$ is simply the number of times that this basis vector was
already chosen. Additionally, in this setting $(\thetahatt)_i$ is just the
average reward observed for the $i$th basis vector, allowing us to
bound the difference between each $(\thetahatt)_i$ and $\theta_i$ using
standard concentration techniques. Using this, we show that with high
probability, after a logarithmic number of rounds---during which the
learner can amass at most $O(1)$ regret since the worst-case regret on
any round is $\varepsilon = 1/\sqrt{T}$---the probability that LinUCB
chooses the wrong action on a $B$ round is small ($O(1/\sqrt{T})$).
This leads to constant regret on expectation.


\OMIT{In this section, we prove that LinUCB achieves constant expected
minority regret when run on the minority alone with
$\theta = \thetazero$. The analysis of the expected minority regret of
LinUCB on the full population for $\theta = \thetaone$ is completely
analogous. We omit formal proofs of the $\Omega(\sqrt{T})$ lower
bounds, which follow a similar structure to the one used in the proof
of the general impossibility result in
Section~\ref{sec:impossibility}, and in fact would follow immediately
as a corollary of this result.}

The proof makes use of the following concentration bound:

\begin{lemma} \label{lem:badrounds}
Let $C_t$ be the number of $C$ rounds observed in the first $t$ minority rounds
in the two-bridge instance. For any $\delta \in (0, 1)$,
with probability at least $1-\delta$, $C_t \ge 0.9t$ for all $t \ge 760 \log
(T/\delta)$.
\end{lemma}
\begin{proof}
We apply the following form of the Chernoff bound:
\[
  \Pr\b{C_t \le (1-\gamma) \E{C_t}} \le \exp\p{-\frac{\gamma^2}{2} \E{C_t}}.
\]
Setting $\gamma = 1/19$, we get
\begin{align*}
  \Pr\b{C_t \le \frac{9}{10}t} &=
  \Pr\b{C_t \le \p{1 - \frac{1}{19}} \frac{19}{20}t}
  \le \exp\p{-\frac{(1/19)^2}{2} \frac{19}{20} t}
  = \exp\p{-\frac{t}{760}}
  \le \frac{\delta}{T}
\end{align*}
for $t \ge 760 \log (T/\delta)$. Applying a union bound over all $T$ rounds, we have $C_t
\ge 0.9t$ for all $t \ge 760 \log(T/\delta)$ with probability at least
$1-\delta$.
\end{proof}

\begin{proofof}[of Theorem~\ref{thm:twobridgelinucb}]
Now, consider LinUCB run on the minority population alone on the
two-bridge instance with $\theta = \thetazero$. Since we are
considering running LinUCB on the minority only, majority rounds are
irrelevant, so throughout this proof we abuse notation and use
$t \in \{1, \cdots, T_0\}$ for some $T_0 \leq T$ to index minority
rounds. $T$ is still the total number of (minority plus majority)
rounds.

  This proof heavily exploits the special structure of the two-bridge
  instance to simplify the analysis of LinUCB.  In particular, we
  exploit the fact that the only contexts ever available are the basis
  vectors $[1 ~ 0]\tran$ and $[0 ~ 1]\tran$.  This implies that the
  covariance matrix $Z_t$ is always diagonal, which greatly simplifies
  the expression for the chosen action in
  Equation~\eqref{eq:linucb_def}.  The optimistic estimate of the
  reward for choosing the $i$th basis vector is simply
  \begin{align}
    UCB_i^t := (\thetahatt)_i + f(t)/\sqrt{(Z_t)_{ii}} .
    \label{eq:ucbdef}
  \end{align}
  Additionally, in this special case,  $(Z_t)_{ii}$ is simply the
  number of times that the $i$th basis vector was chosen over the
  first $t$ minority rounds, and $(\thetahatt)_i$ is the average reward
  observed over the $(Z_t)_{ii}$ rounds on which it was chosen.

  Using this fact, we can apply concentration bounds to bound the
  difference between each $(\thetahatt)_i$ and $\theta_i$. Since rewards
  were assumed to be $1$-subgaussian,
  Lemma~\ref{lem:hoeffding} and a union bound give us that for any
  $\delta_1$, for any $t$, with probability at least $1-4\delta_1$, for all $i
  \in \{1,2\}$,
  \begin{align}
    \left\vert \theta_i - \thetahatti \right\vert \le
        \sqrt{2 \log (\tfrac{1}{\delta_1})/(Z_t)_{ii}}
  \label{eq:ucbbound}
\end{align}

Let $B_t$ and $C_t$ be the number of $B$ and $C$ rounds respectively
before round $t$. By Lemma~\ref{lem:badrounds}, for any $\delta_2$,
with probability $1-\delta_2$, $C_t \ge 9 B_t$ when
$t \ge 760 \log(T/\delta_2)$.  Suppose this is the case.  Since it is
only possible to choose $[1 ~ 0]$ on $B$ rounds, we have
$(Z_t)_{11} \leq B_T$. Similarly, since the algorithm can only choose
$[0 ~ 1]$ on every $C$ round, $(Z_t)_{22} \geq C_T \geq 9 B_T$.
Fixing $\delta_1 = 1/\sqrt{T}$ and using the assumption that
$f(t) \geq 2\sqrt{\log(T)}$, Equations~\eqref{eq:ucbdef}
and~\eqref{eq:ucbbound} then imply that for any $t \geq 760
\log(T/\delta_2)$, with probability at least $1-2\delta_1 = 1 -2\sqrt{T}$,
\begin{align*}
  UCB_1^{t}
  &\ge \theta_1 - \sqrt{\frac{2 \log (\sqrt T)}{(Z_t)_{11}}}+
  \frac{f(t)}{\sqrt{(Z_t)_{11}}}
  \ge \frac{1}{2} + \frac{1}{2}\frac{f(t)}{\sqrt{B_t}} ,
\end{align*}
and similarly,
\begin{align*}
  UCB_2^{t}
  &\le \theta_2 + \sqrt{\frac{2 \log (\sqrt{T})}{(Z_t)_{22}}} +
  \frac{f(t)}{\sqrt{(Z_t)_{22}}}
 \le \frac{1}{2} - \varepsilon   +\frac{3}{2}\frac{f(t)}{\sqrt{C_T}}
\le  \frac{1}{2} +\frac{1}{2}\frac{f(t)}{\sqrt{B_T}} \leq UCB_1^{t}.
\end{align*}

This shows that with probability at least $1-\delta_2$, after the
first $760 \log (T/\delta_2)$ rounds, LinUCB picks $[1 ~ 0]\tran$ on
each $B$ round with probability at least $1-2\delta_1$, leading to
zero regret on that round.  To turn this into a bound on expected
regret, first note that with at most $\delta_2$ probability, the
argument above fails to hold, in which case the minority regret is
still bounded by $\varepsilon B_T \leq \varepsilon T$.  When the argument
above holds, LinUCB may suffer up to $\varepsilon$ regret on each of the
first $760 \log (T/\delta_2)$ minority rounds.  On each additional
round, there is a failure probability of $2\delta_1$, and in this case
LinUCB again suffers regret of at most $\varepsilon$.  Putting this
together and setting $\delta_2 = 1/\sqrt{T}$, we get that the expected
regret is bounded by
$\delta_2 \varepsilon T
+ 760 \log(T/\delta_2) \varepsilon
+ 4\delta_1 \varepsilon T
= O(1)$.
\end{proofof}

\subsection{An Impossibility Result}
\label{sec:impossibility}

It is natural to ask whether it is possible to design an algorithm
that can distinguish between the two scenarios analyzed above,
obtaining minority regret that is close to the best of LinUCB run on
the minority alone and LinUCB run on the full population on any
problem instance. In this section, we show that the answer is no. In
particular, we prove that on the two-bridge instance, if
$\Pr[\theta = \thetazero] = \Pr[\theta = \thetaone] = 1/2$, then any
algorithm must suffer $\Omega(\sqrt{T})$ regret on expectation (and
therefore $\Omega(\sqrt{T})$ minority regret, since all regret is
incurred by minority users).

To prove this result, we begin by formalizing the idea that it is hard
to distinguish between the case in which $\theta = \thetazero$ and the
case in which $\theta = \thetaone$. To do so, we bound the
KL-divergence between the joint distributions over the sequences of
context vectors, actions taken by the given algorithm, and the given
algorithm's rewards that are induced by the two choices of $\theta$.
By applying the high-probability Pinsker lemma~\citep{T09}, we show
that a low KL-divergence between these distributions implies that the
algorithm must be likely either to choose the top bridge on $B$ rounds
more than half the time when the bottom bridge is better or to choose the
bottom bridge on $B$ rounds more than half the time when the top
bridge is better, either of which would lead to high
($\Omega(\sqrt{T})$) regret as long as the number of $B$ rounds is
sufficiently large. To finish the proof, we use a simple Chernoff
bound to show that the number of $B$ rounds is large with high
probability.

To derive the KL-divergence bound, we make use of the assumption that
the realized rewards $r_t$ at each round are either $0$ or $1$.  This
assumption is not strictly necessary.  An analogous argument could be
made, for instance, for real-valued rewards with Gaussian noise.

\begin{theorem} \label{thm:indistinguishability}
On the two-bridge instance with realized rewards $r_t \in \{0,1\}$, any algorithm must incur $\Omega(\sqrt{T})$ minority regret
in expectation when $\Pr[\theta = \thetazero] = \Pr[\theta = \thetaone] = \tfrac12$.
\end{theorem}

Note that ``any algorithm'' here includes algorithms run on the
minority alone, essentially ignoring data from the majority.
Theorems~\ref{thm:twobridgelinucb} and~\ref{thm:indistinguishability}
immediately imply the following corollary.

\begin{corollary}
No algorithm can simultaneously approximate the minority regret of both LinUCB run on the minority and LinUCB run on the full population up to any $o(\sqrt{T})$ multiplicative factor.
\end{corollary}

\begin{proofof}[of Theorem~\ref{thm:indistinguishability}]
Fix any algorithm $\mathcal{A}$. We will first derive an
  $\Omega(\sqrt{T})$ lower bound on the expected regret of
  $\mathcal{A}$ conditioned on the number of
  $B$ rounds, $B_T$, being large. To complete the proof, we then show
  that $B_T$ is large
  with high probability.

  Let $h_t = \{(x_{1,\tau}, x_{2,\tau}, a_\tau, r_\tau)\}_{\tau=1}^t$ be a
  history of all context vectors, chosen actions, and rewards up to
  round $t$, with $h_0 = \emptyset$. Running $\mathcal{A}$ on the
  two-bridge instance with $\theta = \thetazero$ induces a
  distribution over histories $h_T$. Let $P$ denote the conditional
  distribution of these histories, conditioned on the event that $B_T
  \geq T/800$.  That is, we define
\[
P(h_T) := \Pr\left[h_T \left\vert \theta = \thetazero, B_T \geq
    T/800 \right.\right].
\]
  Similarly, we define
\[
Q(h_T) := \Pr\left[h_T \left\vert \theta = \thetaone, B_T \geq
    T/800 \right.\right].
\]

We first show that $\kl{P(h_T)}{Q(h_T)}$ is upper bounded a constant
that does not depend on $T$.
By the
  chain rule for KL divergences, since $r_t$ is independent of any
  previous contexts, actions, or rewards conditioned on $x_t$,
\begin{align*}
  \kl{P(h_T)}{Q(h_T)}
   =
  & \sum_{t=1}^T \Exp_{h_{t-1} \sim P}[\kl{P((x_{1,t}, x_{2,t}, a_t) \given
    h_{t-1})}{Q((x_{1,t}, x_{2,t}, a_t) \given h_{t-1})}] \\
  & + \sum_{t=1}^T \Exp_{(x_{1,t}, x_{2,t}, a_t) \sim P} [\kl{P(r_t \given x_{1,t}, x_{2,t}, a_t)}{Q(r_t
    \given (x_{1,t}, x_{2,t}, a_t))}] .
\end{align*}
Since the choice of context vectors available at time $t$ is
independent of the value of the parameter $\theta$ and $\mathcal{A}$
may only base its choices on the observed history and current choice
of contexts, it is always the case that
$P((x_{1,t}, x_{2,t}, a_t) \given h_{t-1}) = Q((x_{1,t}, x_{2,t}, a_t)
\given h_{t-1})$, so the first sum in this expression is equal to 0.

To bound the second sum, we make use of the assumption that
$r_t \in \{0,1\}$ for all $t$.\footnote{If we instead assumed rewards
  had Gaussian noise with variance $\sigma^2$, we would have
  $\kl{P_t(r_t \given x_{1,t}, x_{2,t}, a_t)}{Q_t(r_t \given x_{1,t}, x_{2,t}, a_t)} =
  \varepsilon^2/(2\sigma^2)$, and the proof would
  still go through.} Lemma~\ref{lem:bern_kl} then tells us that for
any round $t$,
$\kl{P(r_t \given x_{1,t}, x_{2,t}, a_t)}{Q(r_t \given x_{1,t}, x_{2,t}, a_t)} \le 7 \varepsilon^2/2$
since the probability of getting reward 1 conditioned on a chosen
context is always either $1/2$ or $1/2 - \varepsilon$. Putting this
together, we get that \[
  \kl{P(h_T)}{Q(h_T)} \le \frac{7\varepsilon^2 T}{2} = \frac{7}{2}.
\]

Now, let $E$ be the event that the algorithm $\mathcal{A}$ chooses arm
2 on at least half of the $B$ rounds, conditioned on $B_T \ge T/800$. If $E$ occurs when
$\theta = \thetazero$, the regret of $\mathcal{A}$ is at least
$B_T \varepsilon / 2$, which is on the order of $\sqrt{T}$ when
$B_T \geq T /800$. If $E$ does not occur (i.e., $\overline{E}$ occurs)
when $\theta = \thetaone$, $\mathcal{A}$ again has regret at least
$B_T \varepsilon / 2$. We will use the bound on KL
divergence to show that one of these bad cases happens with high
probability.

By Lemma~\ref{lem:pinkser},
\[
  P(E) + Q(\overline{E}) \ge \frac{1}{2} e^{-\kl{P(h_T)}{Q(h_T)}} \geq
  \frac{1}{2} e^{-7/2}.
\]

Let $R$ be the regret of $\mathcal{A}$.  We then have that
\begin{align*}
\Exp\left[R \left\vert B_T \geq \frac{T}{800}\right. \right] &= \frac{1}{2} \Exp\left[R \left\vert \theta = \thetazero, B_T \geq \frac{T}{800}\right. \right]
+ \frac{1}{2} \Exp\left[R \left\vert \theta = \thetaone, B_T \geq \frac{T}{800}\right. \right] \\
&\geq \frac{1}{2} \Pr\left[E \left\vert \theta = \thetazero, B_T \geq \frac{T}{800}\right. \right]
  \Exp\left[R \left\vert E, \theta = \thetazero, B_T \geq \frac{T}{800}\right. \right] \\
& \quad + \frac{1}{2} \Pr\left[\overline{E} \left\vert \theta = \thetaone, B_T \geq \frac{T}{800}\right. \right]
  \Exp\left[R \left\vert \overline{E}, \theta = \thetazero, B_T \geq \frac{T}{800}\right. \right] \\
&\geq \frac{1}{2} \left(P(E) + Q(\overline{E})\right) \frac{\sqrt{T}}{1600} \\
&\geq \frac{\sqrt T e^{-\frac{7}{2}}}{6400}.
\end{align*}

It remains to bound the probability that $B_T \geq T/800$. By a Chernoff bound,
\[
  \Pr\b{B_T \le \frac{T}{800}} = \Pr\b{B_T \le \frac{\E{B_T}}{2}} \le
  \exp\p{-\frac{\E{B_T}}{8}} = \exp\p{-\frac{T}{3200}}.
\]
Thus, for any $\delta \in (0,1)$, if $T \ge 3200 \log(1/\delta)$, then with probability at least $1-\delta$, $B_T \ge
T/800$.  In particular, let $\delta = 1/2$.  Then if
$T \geq 3200 \log 2$, we have
\begin{align*}
\Exp[R]
&\geq
\Pr\left[B_T \geq \frac{T}{800} \right]
\Exp\left[R \left\vert B_T \geq \frac{T}{800}\right. \right]
\geq \left(\frac{1}{2}\right) \left(\frac{\sqrt T e^{-\frac{7}{2}}}{6400}\right).
\end{align*}
This completes the proof that the regret of $\mathcal{A}$ is
$\Omega(\sqrt{T})$ on this problem instance.
\end{proofof}

%% file: sections/bayesian_greedy.tex
We now turn our attention to externalities at an individual level. We interpret exploration as a potential externality imposed on the current user of a system by future users, since the current user would prefer the learner to take the action that appears optimal.  One could avoid such externalities by running the greedy algorithm, which does just that, but it is well known that the greedy algorithm performs poorly in the worst case.  In this section, we build on a recent line of work analyzing the conditions under which inherent data diversity leads the greedy algorithm to perform well.

We analyze the expected performance of the greedy algorithm under small random perturbations of the context vectors. We focus on greedy algorithms that consume new data in batches, rather than every round. We consider both Bayesian and frequentist versions, \bg and \fg. Our main result is that for any specific problem instance, both algorithms match the Bayesian regret of any algorithm on that particular instance up to polylogarithmic factors. We also prove that under the same perturbation assumptions, LinUCB achieves regret $\tilde{O}(T^{1/3})$ for each realization of $\theta$, which implies a worst-case Bayesian regret of $\tilde{O}(T^{1/3})$ for the greedy algorithms. Finally, we repurpose our analysis to derive a positive result in the group setting, implying that the impossibility result of Section~\ref{sec:impossibility} breaks down when the data is sufficiently diverse.

\xhdr{Setting and notation.}
We consider a Bayesian version of linear contextual bandits, with 
$\theta$ drawn from a known
multivariate Gaussian prior
    $\prior = \mc N(\pmt, \pvt)$,
with $\pmt \in \R^d$ and invertible $\pvt \in \R^{d\times d}$.

To capture the idea of data diversity, we assume the context vectors on each round $t$ are generated using the following \emph{perturbed context generation} process:  First, a tuple $(\mu_{1,t} \LDOTS \mu_{K,t} )$ of \emph{mean context vectors} is drawn independently from some fixed distribution $\D_\mu$ over $(\R\cup \{\perp\})^K$, where $\mu_{a, t} = \perp$ means action $a$ is not available. For each available action $a$, the context vector is then $x_{a,t} = \mu_{a,t} + \varepsilon_{a, t}$, where $\varepsilon_{a, t}$ is a vector of random noise. Each component of
    $\varepsilon_{a, t}$
is drawn independently from a zero-mean Gaussian with standard deviation $\rho$.
We refer to $\rho$ as the \emph{perturbation size}. In general, our regret bounds deteriorate if $\rho$ is very small.
Together we refer to a distribution $\D_\mu$, prior $\prior$, and perturbation size $\rho$ as a \textit{problem instance}.

We make several technical assumptions. First, the distribution $\D_\mu$ is such that each context vector has bounded $2$-norm, i.e., $\|\mu_{a,t}\|_2 \le 1$. It can be arbitrary otherwise. Second, the perturbation size needs to be sufficiently small,  $\rho\leq 1/\sqrt{d}$.
Third, the realized reward $r_{a,t}$ for each action $a$ and round $t$ is
    $r_{a,t} = x_{a,t}\tran \theta + \eta_{a,t}$,
the mean reward $x_{a,t}\tran \theta$ plus standard Gaussian noise $\eta_{a,t}\sim \mc N(0, 1)$.\footnote{Our analysis can be easily extended to handle reward noise of fixed variance, i.e.,
    $\eta_{a,t}\sim \mc N(0, \sigma^2)$.
\fg would not need to know $\sigma$. \bg would need to know either $\Sigma$ and $\sigma$ or just $\Sigma/\ \sigma^2$.}

The history up to round $t$ is denoted by the tuple $h_t = ((x_{1}, r_1) \LDOTS (x_t, r_t))$.

\xhdr{The greedy algorithms.}
For the batch version of the greedy algorithm, time is divided in batches of $Y$ consecutive rounds each. When forming its estimate of the optimal action at round $t$, the algorithm may only use the history up to the last round of the previous batch, denoted $t_0$.

\bg forms a posterior over $\theta$ using prior $\prior$ and history $h_{t_0}$.  In round $t$ it chooses the action that maximizes reward in expectation over this posterior. This is equivalent to choosing
\begin{align}\label{eq:BG-est-defn}
 a_t = \argmax_a  x_{a,t}\tran \, \bmt, \quad
    \text{where $\bmt := \Exp[\theta \mid h_{t_0}] $ }.
\end{align}

\fg~does not rely on any knowledge of the prior. It chooses the best action according to the least squares estimate of $\theta$, denoted $\fmt$, computed with respect to history $h_{t_0}$:
\begin{align}\label{eq:FG-est-defn}
 \textstyle a_t = \argmax_a  x_{a, t}\tran \, \fmt, \quad
\text{where $\fmt = \argmin_{\theta'} \sum_{\tau = 1}^{t_0} ((\theta') \tran
x_{\tau} - r_{\tau})^2$}.
\end{align}

\subsection{Main Results}

We first state our main results before describing the intuition behind them. We  state  each theorem  in terms of the main relevant parameters $T$, $K$, $d$, $Y$, and $\rho$.
First, we prove that in expectation over the random perturbations, both greedy algorithms favorably compare to any other algorithm.

\begin{theorem}
With perturbed context generation, there is some
  $Y_0 = \polylog(d, T)/\rho^2$
such that with batch duration $Y\geq Y_0$, the following holds. Fix any bandit algorithm, and let $R_0(T)$ be its Bayesian regret on a particular problem instance. Then on that same instance,
\begin{OneLiners}
\item[(a)] \bg has Bayesian regret at most
    $Y \cdot R_0(T/Y) + \tilde O(1/T)$,
\item[(b)]\fg has Bayesian regret at most
    $Y \cdot R_0(T/Y) + \tilde O(\sqrt{d}/\rho^2).$
\end{OneLiners}
\label{thm:main-greedy}
\end{theorem}

Our next result asserts that the Bayesian regret for LinUCB and both greedy algorithms is on the order of (at most) $T^{1/3}$. This result requires additional technical assumptions.

\begin{theorem}
Assume that the maximal eigenvalue of the covariance matrix $\Sigma$ of the prior $\prior$ is at most $1$,%
\footnote{In particular, if $\prior$ is independent across the coordinates of
$\theta$, then the variance in each coordinate is at most $1$.}
and the mean vector satisfies
    $\|\pmt\|_2\geq 1+\sqrt{3 \log T} $.
With perturbed context generation,
\begin{OneLiners}
\item[(a)] With appropriate parameter settings, LinUCB has Bayesian regret
    $\tilde O(d^2\,K^{2/3}\;T^{1/3}/\rho^2)$.
\item[(b)]If $Y\geq Y_0$ as in Theorem~\ref{thm:main-greedy}, then both \bg and \fg have Bayesian regret at most
    $\tilde O(d^2\,K^{2/3}\;T^{1/3}/\rho^2)$.
\end{OneLiners}
\label{thm:main-worst-case}
\end{theorem}

The assumption $\|\pmt\|_2\geq 1+\sqrt{3 \log T} $ in Theorem~\ref{thm:main-worst-case} can be replaced with $d \ge \log T/\log \log T$.
We use Theorem~\ref{thm:main-worst-case}(b) to derive an ``approximately prior-independent" result for \fg. (For clarity, we state it for independent priors.) The bound in Theorem~\ref{thm:main-worst-case}(b) deteriorates if $\prior$ gets very sharp, but it suffices if $\prior$ has standard deviation on the order of (at least) $T^{-2/3}$.

\begin{corollary}
Assume that the prior $\prior$ is independent over the components of $\theta$, with variance $\kappa^2\leq 1$ in each component. Suppose the mean vector satisfies
    $\|\pmt\|_2\geq 1+\sqrt{3 \log T} $.
With perturbed context generation, if $Y\geq Y_0$ as in Theorem~\ref{thm:main-greedy}, then \fg has Bayesian regret at most
    $\tilde O(d^2\,K^{2/3}\;T^{1/3}/\rho^2)$
as long as $\kappa\geq T^{-2/3}$.
\label{cor:sharp-priors}
\end{corollary}

Finally, we derive a positive result on group externalities. We find that with perturbed context generation, the minority Bayesian regret of the greedy algorithms (i.e., the Bayesian regret incurred on minority rounds) is small compared to the minority Bayesian regret of any algorithm, whether run on the full population \emph{or} on the minority alone. This sidesteps the impossibility result of Section~\ref{sec:impossibility}.

\begin{theorem}
Assume $Y\geq Y_0$ as in Theorem~\ref{thm:main-greedy} and perturbed context generation. Fix any bandit algorithm and instance, and let $R_{\min}(T)$ be the minimum of its minority Bayesian regrets when it is only run over minority rounds or when it is run over the full population. Both greedy algorithms run on the full population achieve minority Bayesian regret at most
    $Y\cdot R_{\min}(T) + \tilde O(\sqrt{d}/\rho^2)$.
  \label{thm:main-greedy-externalities}
\end{theorem}

\subsection{Key Techniques}
\label{sec:bayesian_greedy-key}

The key idea behind our approach is to show that, with perturbed context generation, \BayesGreedy collects data that is informative enough to
``simulate'' the history of contexts and rewards from the run of any other algorithm ALG over fewer rounds. This implies that it remains competitive with ALG since it has at least as much information and makes myopically optimal decisions. \looseness=-1

We use the same technique to prove a similar simulation result for \FreqGreedy. To treat both algorithms at once, we define a template that unifies them. A bandit algorithm is called \emph{\GreedyStyle} if it divides the timeline in batches of Y consecutive rounds each, in each round $t$ chooses some estimate $\theta_t$ of $\theta$, based only on the data from the previous batches, and then chooses the best action according to this estimate, so that
   $a_t = \argmax_a \theta_t\tran x_{a,t}$.
 For a batch that starts at round $t_0+1$, the \emph{batch history} is the tuple
 $((x_{t_0+\tau},\,r_{t_0+\tau}):\; \tau\in [Y])$,
 and the \emph{batch context matrix} is the matrix $X$ whose rows are vectors
 $(x_{t_0+\tau}:\; \tau\in [Y])$;
 here $[Y] = \{1, \cdots, Y\}$. Similarly to the ``empirical covariance matrix",
 we define the \emph{batch covariance matrix} as $X\tran X$.

Let us formulate what we mean by ``simulation". We want to use the data collected from a single batch in order to simulate the reward for any one context $x$. More formally, we are interested in the randomized function that takes a context $x$ and outputs an independent random sample from $\mathcal{N}(\theta\tran x, 1)$. We denote it $\Rew(\cdot)$; this is the realized reward for an action with context vector $x$.

\begin{definition}
Consider batch $B$ in the execution of a \GreedyStyle algorithm. Batch history $h_B$ can simulate $\Rew()$ up to radius $R>0$ if there exists a function
    $g: \{\text{context vectors}\}\times \{ \text{batch histories $h_B$}\} \to \R$
such that $g(x,h_B)$ is identically distributed to $\Rew(x)$ conditional on the batch context matrix, for all $\theta$ and all context vectors $x\in \R^d$ with $\|x\|_2\leq R$.
\label{def:simulation}
\end{definition}

Let us comment on how it may be possible to simulate $\Rew(x)$. For intuition, suppose that
    $x = \tfrac12\, x_1 + \tfrac12\, x_2$.
Then $(\tfrac12\, r_1 + \tfrac12\, r_2 + \xi)$ is distributed
as $\mathcal{N}(\theta\tran x, 1)$ if $\xi$ is drawn independently
from $\mathcal{N}(0, \tfrac12)$. Thus, we can define
    $g(x,h) = \tfrac12\, r_1 + \tfrac12\, r_2 + \xi$
in Definition~\ref{def:simulation}. We generalize this idea and
show that a batch history can simulate $\Rew$ as long as the batch covariance
matrix
has a sufficiently
large minimum eigenvalue, which holds with high probability
when the batch size is large. \looseness=-1

\begin{lemma}
 With perturbed context generation,
  there is some $Y_0 = \polylog(d, T)/\rho^2$ and
  $R = O(\rho \sqrt{d\log(TKd)}) $ such that with probability at least
  $1-T^{-2}$ any batch history from a
  \GreedyStyle algorithm can
    simulate $\Rew()$ up to radius $R$, as long as
    $Y\geq Y_0$.
    \label{lem:simulation}
\end{lemma}

If the batch history of an algorithm can simulate $\Rew$, the algorithm has enough information to simulate the outcome of a fresh round of any other algorithm
$\ALG$. Eventually, this allows us to use a coupling argument in which
we couple a run of \bg with a slowed-down run of $\ALG$, and prove
that the former accumulates at least as much information as the
latter, and therefore the Bayesian-greedy action choice is, in
expectation, at least as good as that of $\ALG$. This leads to
Theorem~\ref{thm:main-greedy}(a). We extend this argument to a
scenario in which both the greedy algorithm and $\ALG$ measure regret
over a randomly chosen subset of the rounds, which leads to
Theorem~\ref{thm:main-greedy-externalities}.

To extend these results to \fg, we consider a hypothetical algorithm that receives the same data as \fg, but chooses actions based on the (batched) Bayesian-greedy selection rule. We analyze this hypothetical algorithm using the same technique as above, and then argue that its Bayesian regret cannot be much smaller than that of \fg. Intuitively, this is because the two algorithms form
almost identical estimates of $\theta$, differing only in the fact that the hypothetical algorithm uses the $\prior$ as well as the data. We show that this difference amounts to effects on the order of $1/t$, which add up to a maximal difference of $O(\log T)$ in Bayesian regret.

%% file: sections/linucb.tex
In this section, we prove Theorem~\ref{thm:main-worst-case}(a), a Bayesian regret bound for the LinUCB algorithm under perturbed context generation. We focus on a version of LinUCB from \citet{abbasi2011improved}, as defined in \eqref{eq:Abbasi-f} on page~\pageref{eq:Abbasi-f}.

Recall that the interval width function in \eqref{eq:Abbasi-f} is parameterized by numbers $L,S,c_0$. We use
\begin{align}
    L &\geq 1 + \rho \sqrt{2d \log(2T^3Kd)}, \nonumber \\
    S &\geq \|\pmt\|_2 + \sqrt{3d\log T}
    \quad \text{(and $S< T$)} \label{eq:LinUCB-params}\\
    c_0 &= 1. \nonumber
\end{align}
Recall that $\rho$ denotes perturbation size, and $\pmt = \E{\theta}$, the prior means of the latent vector $\theta$.

\begin{remark}
Ideally we would like to set $L,S$ according to \eqref{eq:LinUCB-params} with equalities. We consider a more permissive version with inequalities so as to not require the exact knowledge of $\rho$ and $\|\pmt\|_2$.

While the original result in \citet{abbasi2011improved} requires
    $\|x_{a,t}\|_2\leq L$ and $\|\theta\|_2\leq S$,
in our setting this only happens with high probability.
\end{remark}

We prove the following theorem (which implies Theorem~\ref{thm:main-worst-case}(a)):

\begin{theorem}
Assume perturbed context generation. Further, suppose that the maximal eigenvalue of the covariance matrix $\Sigma$ of the prior $\prior$ is at most $1$, and the mean vector satisfies
    $\|\pmt\|_2\geq 1+\sqrt{3 \log T} $.
The version of LinUCB with interval width function \eqref{eq:Abbasi-f} and parameters given by \eqref{eq:LinUCB-params} has Bayesian regret at most
\begin{align}\label{eq:thm:LunUCB-main}
    T^{1/3} \left( d^2\,S\,(K^2/\rho)^{1/3} \right)\cdot \polylog(TKLd).
\end{align}
\label{thm:LunUCB-main}
\end{theorem}

\begin{remark}
The theorem also holds if the assumption on $\|\pmt\|_2$ is replaced with
$d \ge \frac{\log T}{\log \log T}$. The only change in the analysis is that in the concluding steps (Section~\ref{app:linucb-coda}), we use Lemma~\ref{lem:smooth_oful_ex}(b) instead of Lemma~\ref{lem:smooth_oful_ex}(a).
\end{remark}

On a high level, our analysis proceeds as follows. We massage algorithm's regret so as to elucidate the dependence on the number of rounds with small ``gap" between the best and second-best action, call it $N$. This step does not rely on perturbed context generation, and makes use of the analysis from \citet{abbasi2011improved}. The crux is that we derive a much stronger upper-bound on $\E{N}$ under perturbed context generation. The analysis relies on some non-trivial technicalities on bounding the deviations from the ``high-probability" behavior, which are gathered in Section~\ref{app:linucb-deviations}.

We reuse the analysis in \citet{abbasi2011improved} via the following lemma.%
\footnote{Lemma~\ref{lem:reg_sq_bound}(a) is implicit in the proof of Theorem 3 from \citet{abbasi2011improved}, and Lemma~\ref{lem:reg_sq_bound}(b) is asserted by \citet[][Lemma 10]{abbasi2011improved}.}
 To state this lemma,
define the instantaneous regret at time $t$ as
    $\iR{t} = \theta\tran x_t^* - \theta\tran x_{a_t, t} $,
and let
\[
   \beta_T = \p{\sqrt{d \log \p{T(1 + TL^2)}} + S}^2.
  \]

\begin{lemma}[\citet{abbasi2011improved}]
Consider a problem instance with reward noise $\mc N(0, 1)$ and a specific realization of latent vector $\theta$ and contexts $x_{a,t}$. Consider LinUCB with parameters $L,S,c_0$ that satisfy $\|x_{a,t}\|_2\leq L$, $\|\theta\|_2 \le S$, and $c_0= 1$. Then
\begin{OneLiners}
\item[(a)] with probability at least $1-\tfrac{1}{T}$ (over the randomness in the rewards) it holds that
  \[
    \textstyle  \sum_{t=1}^T\; \iR{t}^2 \le 16 \beta_T\; \log(\det(Z_t + I)),
  \]
  where $Z_t =\sum_{\tau=1}^t x_\tau x_\tau\tran\in \R^{d\times d}$ is the ``empirical covariance matrix" at time $t$.
\item[(b)] $\det(Z_t+I) \le (1 +tL^2/d)^d$.
\end{OneLiners}
  \label{lem:reg_sq_bound}
\end{lemma}

The following lemma captures the essence of the proof of Theorem~\ref{thm:LunUCB-main}. From here on, we assume perturbed context generation. In particular, reward noise is $\mc N(0, 1)$.

\begin{lemma}
Suppose parameter $L$ is set as in \eqref{eq:LinUCB-params}. Consider a problem instance with a specific realization of $\theta$ such that $\|\theta\|_2 \le S$.
Then for any $\gamma > 0$,
  \begin{align*}
    \E{\creg{}{T}} \le \|\theta\|_2^{-1/3}\;\p{\frac{1}{2\sqrt{\pi}} + 16 \beta_T\, d
    \log(1 + TL^2/d)} \p{\frac{TK^2}{\rho}}^{1/3} + \tilde
    O\p{1}.
  \end{align*}
  \label{lem:smooth_oful_step}
\end{lemma}

\begin{proof}
We will prove that for any $\gamma > 0$,
  \begin{align}\label{eq:pf:lem:smooth_oful_step}
    \E{\creg{}{T}}
    &\le T \cdot \frac{\gamma^2
    K^2}{2\rho\|\theta\|_2\sqrt{\pi}} + \frac{1}{\gamma} 16 \beta_T\, d
    \log(1 + TL^2/d) + \tilde O(1).
  \end{align}
The Lemma easily follows by setting $\gamma = (TK^2/(\rho \|\theta\|_2))^{-1/3}$.

Fix some $\gamma>0$. We distinguish
between rounds $t$ with $\iR{t}<\gamma $ and those with $\iR{t}\geq \gamma$:
\begin{align}\label{eq:pf:lem:smooth_oful_step:2}
  \creg{}{T} &= \sum_{t=1}^T \iR{t}
  \le \sum_{t \in \mc T_\gamma} \iR{t} + \sum_{t=1}^T
  \frac{\iR{t}^2}{\gamma}
  \le \gamma |\mc T_\gamma| + \frac{1}{\gamma} \sum_{t=1}^T \iR{t}^2,
\end{align}
where $\mc T_\gamma = \{t : \iR{t} \in (0, \gamma)\}$.

We use Lemma~\ref{lem:reg_sq_bound} to upper-bound the second summand in \eqref{eq:pf:lem:smooth_oful_step:2}. To this end, we condition on the event that every
component of every perturbation $\varepsilon_{a, t}$ has absolute value at most $\sqrt{2 \log{2T^3
Kd}}$; denote this event by $U$. This implies $\|x_{a,t}\|_2 \le L$ for all
actions $a$ and all rounds $t$. By Lemma~\ref{lem:subg_union_bound}, $U$ is a high-probability event:
    $\Pr[U] \ge 1 - \frac{1}{T^2}$.
Now we are ready to apply Lemma~\ref{lem:reg_sq_bound}:
\begin{align}\label{eq:pf:lem:smooth_oful_step:3}
\textstyle \E{\sum_{t=1}^T \iR{t}^2 \given U}
    \leq 16\,d\,\beta_T \,\log(1+tL^2/d).
\end{align}
To plug this into \eqref{eq:pf:lem:smooth_oful_step:2}, we need to account for the low-probability event $\bar{U}$. We need to be careful because $R_t$ could, with low probability, be arbitrarily large. By Lemma~\ref{lem:exp_reg_ub_er} with $\ell = 0$,
\begin{align*}
\E{R_t \given \bar U}
    &\le 2\b{\|\theta\|_2 \p{1 + \rho(1+ \sqrt{2 \log K}) + \sqrt{2 \log(2T^3 Kd)}}} \\
\E{\creg{}{T} \given \bar U} \Pr[\bar U]
&= \textstyle \sum_{t=1}^T\;\E{R_t \given \bar U} /T^2 < \tilde O(1). \\
\E{\creg{}{T} \given U}\; \Pr[U]
    &\leq \textstyle \gamma\, \E{\;|\mc T_\gamma|\;}
        + \frac{1}{\gamma} \E{\sum_{t=1}^T \iR{t}^2 \given U}
            &\text{(by \eqref{eq:pf:lem:smooth_oful_step:2})}
\end{align*}
Putting this together and using \eqref{eq:pf:lem:smooth_oful_step:3}, we obtain:
\begin{align}\label{eq:pf:lem:smooth_oful_step:4}
\E{\creg{}{T}}
    \leq
        \gamma\, \E{\;|\mc T_\gamma|\;}
        + \frac{16}{\gamma}\,d\,\beta_T \,\log(1+tL^2/d)+ \tilde O(1).
\end{align}

To obtain \eqref{eq:pf:lem:smooth_oful_step}, we analyze the first summand in \eqref{eq:pf:lem:smooth_oful_step:4}. Let $\Delta_t$ be the ``gap" at time $t$: the difference in expected rewards
  between the best and second-best actions at time $t$ (where ``best" and
  ``second-best" is according to expected rewards). Here, we're taking
  expectations \emph{after} the perturbations are applied, so the only
  randomness comes from the noisy rewards. Consider the set of rounds with small gap,
  $\mc G_\gamma := \{t : \Delta_t < \gamma\}$.
  Notice that $r_t \in (0, \gamma)$ implies $\Delta_t <\gamma$, so
    $|\mc T_\gamma| \le |\mc G_\gamma|$.

In what follows we prove an upper bound on $\E{|\mc G_\gamma|}$. This is the step where perturbed context generation is truly used. For any two arms $a_1$ and $a_2$, the gap between their expected
  rewards is
  \[
    \theta\tran(x_{a_1,t} - x_{a_2,t}) = \theta\tran(\mu_{a_1,t} -
    \mu_{a_2,t}) + \theta\tran(\varepsilon_{a_1,t} - \varepsilon_{a_2,t}).
  \]
  Therefore, the probability that the gap between those arms is smaller than
  $\gamma$ is
  \begin{align*}
    \Pr&\b{|\theta\tran(\mu_{a_1,t} -
    \mu_{a_2,t}) + \theta\tran(\varepsilon_{a_1,t} - \varepsilon_{a_2,t})| \le
    \gamma} \\
    &= \Pr\b{-\gamma - \theta\tran(\mu_{a_1,t} - \mu_{a_2,t}) \le
    \theta\tran(\varepsilon_{a_1,t} - \varepsilon_{a_2,t}) \le \gamma -
    \theta\tran(\mu_{a_1,t} - \mu_{a_2,t})}
  \end{align*}
  Since $\theta\tran\varepsilon_{a_1,t}$ and $\theta\tran\varepsilon_{a_2,t}$
  are both distributed as $\mc N(0, \rho^2 \|\theta\|_2^2)$, their difference is
  $\mc N(0, 2 \rho^2 \|\theta\|_2^2)$. The maximum value that the Gaussian
  measure takes is $\frac{1}{2\rho\|\theta\|_2\sqrt{\pi}}$, and the measure in
  any interval of width $2\gamma$ is therefore at most
  $\frac{\gamma}{\rho\|\theta\|_2\sqrt{\pi}}$. This gives us the bound
  \[
    \Pr\b{|\theta\tran(\mu_{a_1,t} - \mu_{a_2,t}) +
    \theta\tran(\varepsilon_{a_1,t} - \varepsilon_{a_2,t})| \le \gamma} \le
    \frac{\gamma}{\rho\|\theta\|_2\sqrt{\pi}}.
  \]
  Union-bounding over all $\binom{K}{2}$ pairs of actions, we have
\begin{align*}
\Pr[\Delta_t \le \gamma]
    &\le \Pr\b{\bigcup_{a_1, a_2 \in [K]}
          |\theta\tran(x_{a_1,t} - x_{a_2,t})| \le \gamma}
    \le \frac{K^2}{2} \frac{\gamma}{\rho\|\theta\|_2\sqrt{\pi}}.\\
\E{\;|\mc G_\gamma|\;}
    &= \sum_{t=1}^T \Pr[\Delta_t \le \gamma]
    \le T \cdot \frac{K^2}{2} \frac{\gamma}{\rho\|\theta\|_2\sqrt{\pi}}.
\end{align*}
Plugging this into \eqref{eq:pf:lem:smooth_oful_step:4}
  (recalling that
    $|\mc T_\gamma| \le |\mc G_\gamma|$)
  completes the proof.
\end{proof}

\subsection{Bounding the Deviations}
\label{app:linucb-deviations}

We make use of two results that bound deviations from the ``high-probability" behavior, one on $\|\theta\|_2$ and another on instantaneous regret. First, we prove high-probability upper and lower bounds on $\|\theta\|_2$ under the conditions in Theorem~\ref{thm:LunUCB-main}. Essentially, these bounds allow us to use Lemma~\ref{lem:smooth_oful_step}.

\begin{lemma}\label{lem:smooth_oful_ex}
Assume the latent vector $\theta$ comes from a multivariate Gaussian,
    $\theta \sim \mc N(\pmt, \pvt)$,
here the covariate matrix $\pvt$ satisfies $\lambda_{\max}(\pvt) \le 1$.
\begin{itemize}
\item[(a)] If $\|\pmt\|_2 \ge 1+\sqrt{3\log T}$,  then
  for sufficiently large $T$, with probability at least $1-\frac{2}{T}$,
  \begin{align}\label{eq:lem:smooth_oful_ex}
    \tfrac{1}{2\log T} \le \|\theta\|_2 \le \|\pmt\|_2 + \sqrt{3d \log T}.
  \end{align}
\item[(b)] Same conclusion if $d \ge \frac{\log T}{\log \log T}$.
\end{itemize}
\end{lemma}
\begin{proof}
  We consider two cases, based on whether $d \ge \log T/\log \log T$. We need both cases to prove part (a), and we obtain part (b) as an interesting by-product.
We repeatedly use
  Lemma~\ref{lem:chi_sq_conc}, a concentration inequality for $\chi^2$ random
  variables, to show concentration on the Gaussian norm.

  \textbf{Case 1:} $d \ge \log T/\log \log T$. \\
  Since the Gaussian measure is decreasing in
  distance from 0, the $\Pr\b{\|\theta\|_2 \le c} \le \Pr\b{\|\pmt - \theta\|_2
  \le c}$ for any $c$. In other words, the norm of a Gaussian is most likely to
  be small when its mean is 0. Let $X = \pvt^{-1/2} (\pmt - \theta)$. Note that
  $X$ has distribution $\mc N(0, I)$, and therefore $\|X\|_2^2$ has $\chi^2$
  distribution with $d$ degrees of freedom. We can bound this as
  \begin{align*}
    \Pr\b{\|\pmt - \theta\|_2 \le \frac{1}{2\log T}}
    &= \Pr\b{\|\pvt^{-1/2}X\|_2 \le \frac{1}{2\log T}} \\
    &\le \Pr\b{\sqrt{\lambda_{\max}(\pvt)}\|X\|_2 \le \frac{1}{2\log T}} \\
    &\le \Pr\b{\|X\|_2 \le \frac{1}{2\log T}} \\
    &= \Pr\b{\|X\|_2^2 \le \frac{1}{4(\log T)^2}} \\
    &\le \p{\frac{1}{4d(\log T)^2} e^{1-1/((4\log T)^2 d)}}^{d/2} \tag{By
      Lemma~\ref{lem:chi_sq_conc}} \\
    &\le \p{\frac{\log \log T}{(\log T)^3}}^{\log T/(2\log \log T)} \tag{$d \ge
    \log T/\log \log T$} \\
    &= \frac{T^{\log \log \log T/(2 \log \log T)}}{T^{3/2}} \\
    &\le T^{-1}
  \end{align*}
  Similarly, we can show
  \begin{align*}
    \Pr\b{\|\pmt - \theta\|_2 \ge \sqrt{d\log T}}
    &= \Pr\b{\|\pvt^{-1/2}X\|_2 \ge \sqrt{d \log T}} \\
    &\le \Pr\b{\sqrt{\lambda_{\max}(\pvt)}\|X\|_2 \ge \sqrt{d \log T}} \\
    &\le \Pr\b{\|X\|_2 \ge \sqrt{d\log T}} \\
    &= \Pr\b{\|X\|_2^2 \ge d \log T} \\
    &\le \p{\log T e^{1-\log T}}^{d/2} \tag{By Lemma~\ref{lem:chi_sq_conc}} \\
    &\le \p{\exp\p{1 + \log \log T - \log T}}^{\log T/(2\log \log T)} \tag{$d
    \ge \log T/\log \log T$} \\
    &= T^{(1 + \log \log T - \log T)/(2\log \log T)} \\
    &\le T^{-1} \tag{For $\log T > 1 + 3 \log \log T$}
  \end{align*}
  By the triangle inequality,
  \[
    \|\pmt\|_2 - \|\pmt - \theta\|_2 \le \|\theta\|_2 \le \|\pmt\|_2 + \|\pmt -
    \theta\|_2.
  \]
  Thus, in this case, $\frac{1}{2\log T} \le \|\theta\|_2 \le \|\pmt\|_2 +
  \sqrt{d \log T}$ with probability at least $1-2T^{-1}$.

  \textbf{Case 2:} $\|\pmt\|_2 \ge 1 + \sqrt{3 \log T}$ and $d < \log T/\log
  \log T$. \\
  For this part of the proof, we just need that $d < \log T$, which it is by
  assumption. Using the triangle inequality, if $\|\pmt\|_2$ is large, it
  suffices to show that $\|\pmt - \theta\|_2$ is small with high probability.
  Again, let $X = \pvt^{-1/2} (\pmt - \theta)$. Then,
  \begin{align*}
    \Pr\b{\|\pmt - \theta\|_2 \ge \sqrt{3 \log T}}
    &= \Pr\b{\|\pvt^{1/2} X\|_2 \ge \sqrt{3 \log T}} \\
    &\ge \Pr\b{\sqrt{\lambda_{\max}(\pvt)} \|X\|_2 \ge \sqrt{3 \log T}} \\
    &= \Pr\b{\|X\|_2 \ge \frac{\sqrt{3 \log T}}{\sqrt{\lambda_{\max}(\pvt)}}} \\
    &\ge \Pr\b{\|X\|_2 \ge \sqrt{3 \log T}} \\
    &= \Pr\b{\|X\|_2^2 \ge 3 \log T}
  \end{align*}
  By Lemma~\ref{lem:chi_sq_conc},
  \begin{align*}
    \Pr\b{\|X\|_2^2 \ge 3 \log T}
    &\le \p{\frac{3\log T}{d} e^{1-\frac{3\log T}{d}}}^{d/2} \\
    &= \p{T^{-3/d}e \frac{3\log T}{d}}^{d/2} \\
    &= T^{-1} \p{T^{-1/d}e \frac{3\log T}{d}}^{d/2} \\
    &\le T^{-1} \tag{for sufficiently large
    $T$} \end{align*}
  Because $\|\pmt\|_2 \ge 1 + \sqrt{3 \log T}$, $1 \le \|\theta\|_2 \le
  \|\pmt\|_2 + \sqrt{3 \log T}$ with probability at least $1-T^{-1}$.
\end{proof}

Next, we show how to upper-bound expected instantaneous regret in the worst case.%
\footnote{We state and prove this result in a slightly more general version which we use to support Section~\ref{sec:bayesian_greedy}. For the sake of this section, a special case of $\ell=0$ suffices.}

\begin{lemma}
Fix round $t$ and parameter $\ell>0$. For any $\theta$, conditioned on any history $h_{t-1}$ and the event that
  $\|\varepsilon_{a,t}\|_\infty \ge \ell$
for each arm $a$, the expected instantaneous regret of
  any algorithm at round $t$ is at most
  \[
    2\, \|\theta\|_2\p{1 + \rho(2 + \sqrt{2 \log K}) + \ell}.
  \]
  \label{lem:exp_reg_ub_er}
\end{lemma}
\begin{proof}
  The expected regret at round $t$ is upper-bounded by the reward difference
  between the best arm $x_t^*$ and the worst arm $x_t^\dagger$, which is
  \[
    \theta\tran (x_t^* - x_t^\dagger).
  \]
  Note that $x_t^* = \mu_t^* + \varepsilon_t^*$ and $x_t^\dagger = \mu_t^\dagger
  + \varepsilon_t^\dagger$. Then, this is
  \begin{align*}
    \theta\tran (x_t^* - x_t^\dagger) &=
    \theta\tran (\mu_t^* - \mu_t^\dagger) + \theta\tran (\varepsilon_t^* -
    \varepsilon_t^\dagger) \\
    &\le 2\|\theta\|_2 + \theta\tran (\varepsilon_t^* - \varepsilon_t^\dagger)
  \end{align*}
  since $\|\mu_{a,t}\|_2 \le 1$. Next, note that
  \[
    \theta\tran \varepsilon_t^* \le \max_a \theta\tran \varepsilon_{a,t}
  \]
  and
  \[
    \theta\tran \varepsilon_t^\dagger \ge \min_a \theta\tran \varepsilon_{a,t}.
  \]
  Since $\varepsilon_{a,t}$ has symmetry about the origin conditioned on the
  event that at least one component of one of the perturbations has absolute
  value at least $\ell$, i.e. $v$ and $-v$ have equal likelihood, $\max_a
  \theta\tran \varepsilon_{a,t}$ and $-\min_a \theta\tran \varepsilon_{a,t}$ are
  identically distributed. Let $\elt$ be the event that at least one of the
  components of one of the perturbations has absolute value at least $\ell$.
  This means for any choice $\mu_{a,t}$ for all $a$,
  \begin{align*}
    \Exp \b{\theta\tran (x_t^* - x_t^\dagger) \given \elt}
    &\le 2\|\theta\|_2 + 2 \Exp \b{\max_a \theta\tran
    \varepsilon_{a,t} \given \elt}
  \end{align*}
  where the expectation is taken over the perturbations at time $t$.

  Without loss of generality, let $(\varepsilon_{a',t})_j$ be the component such
  that $|(\varepsilon_{a',t})_j| \ge \ell$. Then, all other components have
  distribution $\mc N(0, \rho^2)$. Then,
  \begin{align*}
    \Exp \b{\max_a \theta\tran \varepsilon_{a,t} \given
    \elt}
    &=\Exp \b{\max_a \theta\tran \varepsilon_{a,t} \given
    |(\varepsilon_{a',t})_j| \ge \ell} \\
    &=\Exp \b{\max(\theta\tran \varepsilon_{a',t}, \max_{a
    \ne a'} \theta\tran \varepsilon_{a ,t}) \given |(\varepsilon_{a',t})_j| \ge
    \ell} \\
    &\le\Exp \b{\max\p{|\theta_j (\varepsilon_{a',t})_j| +
    \sum_{i \ne j} \theta_i (\varepsilon_{a',t})_i, \max_{a \ne a'} \theta\tran
    \varepsilon_{a ,t}} \given |(\varepsilon_{a',t})_j| \ge \ell}
  \end{align*}
  Let $(\tilde \varepsilon_{a,t})_i = 0$ if $a = a'$ and $i = j$, and
  $(\varepsilon_{a,t})_i$ otherwise. In other words, we simply zero out the
  component $(\varepsilon_{a',t})_j$. Then, this is
  \begin{align*}
    &\Exp \b{\max\p{|\theta_j (\varepsilon_{a',t})_j| +
    \theta\tran \tilde \varepsilon_{a',t}, \max_{a \ne a'} \theta\tran
    \tilde \varepsilon_{a ,t}} \given |(\varepsilon_{a',t})_j| \ge \ell} \\
    &\le \Exp \b{\max_a \p{|\theta_j
    (\varepsilon_{a',t})_j| + \theta\tran \tilde \varepsilon_{a,t}} \given
    |(\varepsilon_{a',t})_j| \ge \ell} \\
    &= \Exp \b{|\theta_j (\varepsilon_{a',t})_j| + \max_a
    \p{\theta\tran \tilde \varepsilon_{a,t}} \given |(\varepsilon_{a',t})_j|
    \ge \ell} \\
    &= \Exp \b{|\theta_j (\varepsilon_{a',t})_j|\given
    |(\varepsilon_{a',t})_j| \ge \ell} + \Exp \b{\max_{a}
    \p{\theta\tran \tilde \varepsilon_{a,t}}} \\
    &\le \Exp \b{|\theta_j (\varepsilon_{a',t})_j|\given
    |(\varepsilon_{a',t})_j| \ge \ell} + \rho \|\theta\|_2\sqrt{2\log K}
  \end{align*}
  because by Lemma~\ref{lem:subgaussian_max},
  \[
    \Exp \b{\max_a \theta\tran \tilde \varepsilon_{a,t}} \le
    \Exp \b{\max_a \theta\tran \varepsilon_{a,t}} \le \rho
    \|\theta\|_2 \sqrt{2 \log K}
  \]
  Next, note that by symmetry and since $\theta_j \le \|\theta\|_2$,
  \[
    \Exp \b{|\theta_j (\varepsilon_{a',t})_j|\given
    |(\varepsilon_{a',t})_j| \ge \ell}
    \le \|\theta\|_2 \Exp \b{(\varepsilon_{a',t})_j\given
    (\varepsilon_{a',t})_j \ge \ell}.
  \]
  By Lemma~\ref{lem:gaus_exp_bound},
  \[
    \Exp \b{(\varepsilon_{a',t})_j\given
    (\varepsilon_{a',t})_j \ge \ell} \le \max(2\rho, \ell + \rho)
    \le 2\rho + \ell
  \]
  Putting this all together, the expected instantaneous regret is bounded by
  \[
    2\p{\|\theta\|_2\p{1 + \rho(2 + \sqrt{2 \log K}) + \ell}},
  \]
  proving the lemma.
\end{proof}

\subsection{Finishing the Proof of Theorem~\ref{thm:LunUCB-main}.}
\label{app:linucb-coda}

We focus on the ``nice event" that \eqref{eq:lem:smooth_oful_ex} holds, denote it
$ \mE$ for brevity. In particular, note that it implies $\|\theta\|_2\leq S$. Lemma~\ref{lem:smooth_oful_step} guarantees that expected regret under this event,
    $\E{\creg{}{T} \given \mE}$, is upper-bounded by
the expression \eqref{eq:thm:LunUCB-main} in the theorem statement.

In what follows we use Lemma~\ref{lem:smooth_oful_ex}(a) and Lemma~\ref{lem:exp_reg_ub_er} guarantee that if $\mE$ fails, then the
corresponding contribution to expected regret is small. Indeed, Lemma~\ref{lem:exp_reg_ub_er} with $\ell = 0$ implies that
\begin{align*}
    \E{R_t \given \bar{\mE}\,}
       \leq BT\,\|\theta\|_2 \quad\text{for each round $t$},
\end{align*}
where $B= 1 + \rho(2 + \sqrt{2 \log K})$ is the ``blow-up factor". Since  \eqref{eq:lem:smooth_oful_ex} fails with probability at most $\tfrac{2}{T}$
by Lemma~\ref{lem:smooth_oful_ex}(a), we have
\begin{align*}
\E{\creg{}{T} \given \bar{\mE}\,} \;\Pr[\bar{\mE}\,]
    & \leq \tfrac{2B}{T}\; \E{ \|\theta\|_2 \given \bar{\mE}\,} \\
    & \leq \tfrac{2B}{T}\;
        \E{ \|\theta\|_2 \given \|\theta\|_2 \geq \tfrac{1}{2\log T}\,} \\
    &\leq O\p{\tfrac{B}{T}}\; \p{\|\pmt\|_2 + d\log T} \\
    &\leq O(1).
\end{align*}

The antecedent inequality follows by Lemma~\ref{lem:gaus_exp_norm_bound} with
            $\alpha =\tfrac{1}{2\log T} $,
using the assumption that $\lambda_{\max}(\Sigma)\leq 1$. The theorem follows.

%% file: sections/bg_proofs.tex
We present the proofs for our results on greedy algorithms in Section~\ref{sec:bayesian_greedy}.%
\footnote{That is, all results in Section~\ref{sec:bayesian_greedy} except the regret bound for LinUCB (Theorem~\ref{thm:main-worst-case}(a)), which is proved in Section~\ref{app:linucb}.} This section is structured as follows. In Section~\ref{app:pf_bg:diversity}, we quantify the diversity of data collected by \GreedyStyle algorithms, assuming perturbed context generation. In Section~\ref{app:pf_bg:simulation}, we show that a sufficiently ``diverse" batch history suffices to simulate the reward for any given context vector, in the sense of Definition~\ref{def:simulation}. Jointly, these two subsections imply  that batch history generated by a \GreedyStyle algorithm can simulate rewards with high probability, as long as the batch size is sufficiently large. Section~\ref{sec:bg-proofs-bg} builds on this foundation to derive regret bounds for \bg. The crux is that the history collected by \bg suffices to simulate a ``slowed-down" run of any other algorithm. This analysis extends to a version of \fg equipped with a Bayesian-greedy prediction rule (and tracks the performance of the prediction rule). Finally, Section~\ref{sec:bg-proofs-fg} derives the regret bounds for \fg, by comparing the prediction-rule version of \fg with \fg itself. To derive the results on group externalities, we present all our analysis in Sections~\ref{sec:bg-proofs-bg} and~\ref{sec:bg-proofs-fg} in a more general framework in which only the minority rounds are counted for regret.

\xhdr{Preliminaries.}
We assume perturbed context generation in this section, without further mention.

Throughout, we will use the following parameters as a shorthand:
\begin{align*}
\delta_R &= T^{-2} \\
\hat R  & = \rho\sqrt{2 \log(2 TKd/\delta_R)} \\
R   &= 1 + \hat{R} \sqrt{d}.
\end{align*}
Recall that $\rho$ denotes perturbation size, and $d$ is the dimension. The meaning of $\hat R$ and $R$ is that they are high-probability upper bounds on the perturbations and the contexts, respectively. More formally, by Lemma~\ref{lem:subgaussian_max} we have:
\begin{align}
\Pr\left[ \|\eps_{a,t}\|_\infty \leq \hat R:\;
    \text{ for all arms $a$ and all rounds $t$ } \right] &\leq \delta_R
    \label{eq:bg-proofs-highprob-R-hat}\\
\Pr\left[ \|x_{a,t}\|_2 \leq R:\;
    \text{ for all arms $a$ and all rounds $t$ } \right] &\leq \delta_R
    \label{eq:bg-proofs-highprob-R}
\end{align}

Let us recap some of the key definitions from Section~\ref{sec:bayesian_greedy-key}. We consider \GreedyStyle algorithms, a template that unifies \BayesGreedy and \FreqGreedy. A bandit algorithm is called \emph{\GreedyStyle} if it divides the timeline in batches of Y consecutive rounds each, in each round $t$ chooses some estimate $\theta_t$ of $\theta$, based only on the data from the previous batches, and then chooses the best action according to this estimate, so that
 $a_t = \argmax_a \theta_t\tran x_{a,t}$.

For a batch $B$ that starts at round $t_0+1$, the \emph{batch history} $h_B$ is the tuple
    $((x_{t_0+\tau},\,r_{t_0+\tau}):\; \tau\in [Y])$,
and the \emph{batch context matrix} $X_B$ is the matrix whose rows are vectors
    $(x_{t_0+\tau}:\; \tau\in [Y])$.
Here and elsewhere, $[Y] = \{1, \cdots, Y\}$. The \emph{batch covariance matrix} is defined as
\begin{align}\label{eq:ZB-defn}
\ZB := X_B\tran\, X_B = \sum_{t=t_0+1}^{t_0+Y} x_t\, x_t\tran.
\end{align}

\subsection{Data Diversity under Perturbations}
\label{app:pf_bg:diversity}

We are interested in the diversity of data collected by \GreedyStyle algorithms, assuming perturbed context generation. Informally, the observed contexts
    $x_1, x_2,\,\ldots $
should cover all directions in order to enable good estimation of the latent vector $\theta$. Following \citet{kannan2018smoothed}, we quantify data diversity via the minimal eigenvalue of the empirical covariance matrix $Z_t$. More precisely, we are interested in proving that $\lambda_{\min}(Z_t)$ is sufficiently large. We adapt some tools from \citet{kannan2018smoothed}, and then derive some improvements for \GreedyStyle algorithms.

\subsubsection{Tools from~\citet{kannan2018smoothed}}

\citet{kannan2018smoothed} prove that $\lambda_{\min}(Z_t)$ grows linearly in time $t$, assuming $t$ is sufficiently large.

\begin{lemma}[\citet{kannan2018smoothed}]
Fix any \GreedyStyle algorithm. Consider round $t \geq \tau_0$, where
    $\tau_0 =160 \frac{R^2}{\rho^2} \log \frac{2d}{\delta} \cdot \log T$.
Then for any realization of $\theta$, with probability $1-\delta$
  \[
    \lambda_{\min}(Z_t) \ge \frac{\rho^2 t}{32 \log T}.
  \]
  \label{lem:fg_big_cov}
\end{lemma}
\begin{proof}
  The claimed conclusion follows from an argument inside the proof  of
    Lemma B.1 from \citet{kannan2018smoothed},
  plugging in
  $  \lambda_0 = \frac{\rho^2}{2\log T}$.
  This argument applies for any $t\geq \tau'_0$, where
    $\tau'_0 = \max\p{32 \log \frac{2}{\delta}, 160 \frac{R^2}{\rho^2}
  \log \frac{2d}{\delta} \cdot \log T}$.
We observe that $\tau'_0=\tau_0$ since $R \ge \rho$.
\end{proof}

Recall that $Z_t$ is the sum
    $Z_t :=\sum_{\tau=1}^t x_\tau x_\tau\tran$.
A key step in the proof of Lemma~\ref{lem:fg_big_cov} zeroes in on the expected contribution of a single round. We use this tool separately in the proof of Lemma~\ref{lem:min_ev_bg}.

\begin{lemma}[\citet{kannan2018smoothed}]\label{lem:eig_increase}
Fix any \GreedyStyle algorithm, and the latent vector $\theta$. Assume $T \ge 4K$. Condition on the event that all perturbations $\eps_{a,t}$ are upper-bounded by
$\hat R$, denote it with $\mE$.
Then with probability at least $\tfrac14$,
  \[
    \lambda_{\min}\p{\E{x_t x_t\tran \given h_{t-1}, \mE}} \ge \frac{\rho^2}{2\log T}.
  \]
\end{lemma}
\begin{proof}
The proof is easily assembled from several pieces in the analysis in  \citet{kannan2018smoothed}.
Let $\thetahatt$ be the algorithm's estimate for $\theta$ at time
$t$. As in~\citet{kannan2018smoothed}, define
\[
  \hcat = \max_{a' \ne a} \thetahatt\tran x_{a',t},
\]
where $\hcat$ depends on all perturbations other than the perturbation for
$x_{a,t}$. Let us say that $\hcat$ is ``good'' for arm $a$ if
\[
  \hcat \le \thetahatt\tran \mu_{a,t} + \rho \sqrt{2 \log T}
  \|\thetahatt\|_2.
\]

First we argue that
  \begin{align}\label{eq:pf:lem:eig_increase-1}
    \Pr\b{\hcat \text{ is good for } a \given a_t = t, \mE} \ge \tfrac14.
  \end{align}

Indeed, in the proof of their Lemma~3.4, \citet{kannan2018smoothed} show that for
any round, conditioned on $\mE$, if the probability that arm $a$ was chosen over the randomness of the
perturbation is at least $2/T$, then
  the round is good for $a$ with probability at least $\tfrac12$. Let $B_t$ be the set of
  arms at round $t$ with probability at most $2/T$ of being chosen over the
  randomness of the perturbation. Then,
  \[
    \Pr_{\varepsilon \sim \mc N(0, \rho^2 I)} \b{a_t \in B_t} \le \sum_{a \in
    B_T} \Pr_{\varepsilon \sim \mc N(0, \rho^2 I)} \b{a_t = a} \le
    \tfrac{2}{T} |B_t| \le \tfrac{2K}{T} \le \tfrac{1}{2}.
  \]
  Since by assumption $T \ge 4K$, \eqref{eq:pf:lem:eig_increase-1} follows.

Second, we argue that
\begin{align}\label{eq:pf:lem:eig_increase-2}
    \lambda_{\min}\p{\E{x_{a,t}x_{a,t}\tran \given a_t=a, \hcat \text{ is
    good}}} \ge \frac{\rho^2}{2\log T}
\end{align}
This is where we use conditioning on the event $\{ \eps_{a,t} \leq \hat R\}$. We plug in $r = \rho \sqrt{2 \log T}$ and $\lambda_0 = \frac{\rho^2}{2 \log T}$
into Lemma 3.2 of \citet{kannan2018smoothed}. This lemma applies because with these parameters, the perturbed distribution of context arrivals satisfies the
    $(\rho\sqrt{2 \log T}, \rho^2/(2 \log T))$-diversity
condition from \citet{kannan2018smoothed}. The latter is by
    Lemma 3.6 of \citet{kannan2018smoothed}.
This completes the proof of \eqref{eq:pf:lem:eig_increase-2}. The lemma follows from \eqref{eq:pf:lem:eig_increase-1} and \eqref{eq:pf:lem:eig_increase-2}.
\end{proof}

Let $\fmt$ be the \FreqGreedy estimate for $\theta$ at time $t$, as defined in \eqref{eq:FG-est-defn}. We are interested in quantifying how the quality of this estimate improves over time. \citet{kannan2018smoothed} prove, essentially, that the distance between $\fmt$ and $\theta$ scales as
    $\sqrt{t}/\lambda_{\min}(Z_t)$.
\begin{lemma}[\citet{kannan2018smoothed}]
Consider any round $t$ in the execution of \FreqGreedy. Let $t_0$ be the last round of the previous batch. For any $\theta$ and any $\delta>0$, with probability $1-\delta$,
  \[
    \|\theta - \fmt\|_2 \le
    \frac{\sqrt{t_0 \cdot 2dR \log \tfrac{d}{\delta}}}{\lambda_{\min}(Z_{t_0})}.
  \]
  \label{lem:fmt_close}
\end{lemma}

\subsubsection{Some improvements}

We focus on batch covariance matrix $\ZB$ of a given batch in a \GreedyStyle algorithm. We would like to prove that $\lambda_{\min}(\ZB)$ is sufficiently large with high probability, as long as the batch size $Y$ is large enough. The analysis from \citet{kannan2018smoothed} (a version of Lemma~\ref{lem:fg_big_cov}) would apply, but only as long as the batch size is least as large as the $\tau_0$ from the statement of Lemma~\ref{lem:fg_big_cov}. We derive a more efficient version, essentially shaving off a factor of $8$.%
\footnote{Essentially, the factor of $160$ in Lemma~\ref{lem:fg_big_cov} is replaced with factor $\tfrac{8e^2}{(e-1)^2}<20.022$ in \eqref{eq:lem:min_ev_bg-Y}.}

\begin{lemma}
Fix a \GreedyStyle algorithm and any batch $B$ in the execution of this algorithm.  Fix $\delta>0$ and assume that the batch size $Y$ is at least
\begin{align}\label{eq:lem:min_ev_bg-Y}
    Y_0 :=  (\tfrac{R}{\rho})^2 \,
            \tfrac{8e^2}{(e-1)^2}\,
            \p{1 + \log \tfrac{2d}{\delta}}\, \log(T)
    + \tfrac{4e}{e-1} \log \tfrac{2}{\delta}.
\end{align}
Condition on the event that all perturbations in this batch are upper-bounded by $\hat R$, more formally:
\[ \mE_B = \{ \|\eps_{a,t}\|_\infty \leq \hat R:\;
    \text{ for all arms $a$ and all rounds $t$ in $B$} \}. \]
Further, condition on the latent vector $\theta$ and the history $h$ before batch $B$. Then
\begin{align}\label{eq:lem:min_ev_bg}
    \Pr\left[\;  \lambda_{\min}(\ZB) \ge R^2 \given \mE_B,h,\theta \right] \geq 1-\delta.
\end{align}
The probability in \eqref{eq:lem:min_ev_bg} is over the randomness in context arrivals and rewards in batch $B$.
  \label{lem:min_ev_bg}
\end{lemma}

The improvement over Lemma~\ref{lem:fg_big_cov} comes from two sources: we use a tail bound on the sum of
  geometric random variables instead of a Chernoff bound on a binomial random
  variable, and we derive a tighter application of the eigenvalue concentration
inequality of~\citet{tropp2012user}.

\begin{proof}
Let $t_0$ be the last round before batch $B$. Recalling \eqref{eq:ZB-defn}, let
\[ \WB = \sum_{t=t_0+1}^{t_0+Y} \E{x_t x_t\tran \given h_{t-1}} \]
be a similar sum over the expected per-round covariance matrices. Assume $Y\geq Y_0$

The proof proceeds in two steps: first we lower-bound $\lambda_{\min}(\ZB)$, and then we show that it implies \eqref{eq:lem:min_ev_bg}. Denoting
    $ m = R^2\,\tfrac{e}{e-1}\,(1+\log \tfrac{2d}{\delta})$,
we claim that
\begin{align}\label{eq:matrix_conc}
    \Pr\b{\lambda_{\min}(\WB) < m \given \mE_B,h} \le \tfrac{\delta}{2}.
\end{align}

To prove this, observe that $\WB$'s minimum eigenvalue increases by at least $\lambda_0
  = \rho^2/(2\log T)$ with probability at least $1/4$ each round by
  Lemma~\ref{lem:eig_increase}, where the randomness is over the history, \ie the
  sequence of (context, reward) pairs. If we want it to go up to $m$, this
  should take $4m/\lambda_0$ rounds in expectation. However, we need it to go to
  $m$ with high probability. Notice that this is dominated by the sum of
  $m/\lambda_0$ geometric random variables with parameter $\frac{1}{4}$. We'll
  use the following bound from~\citet{janson2017tail}: for $X = \sum_{i=1}^n X_i$
  where $X_i \sim \text{Geom}(p)$ and any $c \ge 1$,
  \[
    \Pr[X \ge c\E{X}] \le \exp\p{-n(c - 1 - \log c)}.
  \]
  Because we want the minimum eigenvalue of $\WB$ to be $m$, we need $n =
  m/\lambda_0$, so $\E{X} = 4m/\lambda_0$. Choose  $c =
  \p{1+\frac{\lambda_0}{m} \log \tfrac{2}{\delta}} \tfrac{e}{e-1}$. By
  Corollary~\ref{cor:e1ex},
  \begin{align*}
    c - 1 - \log c &\ge \tfrac{e-1}{e} \cdot c - 1 = \tfrac{\lambda_0}{m} \log \tfrac{2}{\delta}.
  \end{align*}
  Therefore,
  \[
    \Pr\b{X \ge c\E{X}} \le \exp\p{-n \cdot \tfrac{\lambda_0}{m} \log
    \tfrac{2}{\delta}} = \p{\tfrac{\delta}{2}}^{n \cdot \lambda_0/m} =
    \tfrac{\delta}{2}
  \]
  Thus, with probability $1 - \frac{\delta}{2}$, $\lambda_{\min}(\WB) \ge m$ as long as the batch size $Y$ is at least
  \[
    \frac{e}{e-1} \p{1 + \frac{\lambda_0}{m} \log \frac{2}{\delta}} \cdot
    \E{X} = \frac{4e}{e-1} \p{\frac{m}{\lambda_0} + \log \frac{2}{\delta}} = Y_0.
  \]
  This completes the proof of \eqref{eq:matrix_conc}.

To derive \eqref{eq:lem:min_ev_bg} from \eqref{eq:matrix_conc}, we proceed as follows. Consider the event
\[ \mE = \left\{\;  \lambda_{\min}(\ZB) \le R^2 \text{ and }
             \lambda_{\min}(\WB) \ge m \; \right\}. \]
Letting $\alpha = 1 - R^2/m$ and rewriting $R^2$ as $(1-\alpha)m$, we use a concentration inequality from \citet{tropp2012user} to guarantee that
  \[
    \Pr[\mE \given \mE_B,h]
        \le d \p{e^\alpha (1-\alpha)^{1-\alpha}}^{-m/R^2}.
  \]
  Then, using the fact that $x^x \ge e^{-1/e}$ for all $x > 0$, we
  have
\begin{align*}
\Pr[\mE \given \mE_B,h]
    &\le d\p{e^{1-R^2/m-1/e}}^{-m/R^2}
    = d\, e^{-(m-R^2-m/e)/R^2} \\
    &= d \exp\p{-\frac{\p{\frac{e-1}{e}}m}{R^2} + 1}
    \le \tfrac{\delta}{2},
  \end{align*}
  since $m \ge \frac{e}{e-1} R^2 \p{1+\log \frac{2d}{\delta}}$.
  Finally, observe that, omitting the conditioning on $\mE_B,h$, we have:
  \[
    \Pr\b{\lambda_{\min}(\ZB) \leq R^2 }
        \le \Pr\b{\mE} + \Pr\b{\lambda_{\min}(\WB) < m} \le
    \tfrac{\delta}{2} + \tfrac{\delta}{2} = \delta.
  \]
\end{proof}


\subsection{Reward Simulation with a Diverse Batch History}
\label{app:pf_bg:simulation}

We consider reward simulation with a batch history, in the sense of
Definition~\ref{def:simulation}. We show that a sufficiently ``diverse" batch
history suffices to simulate the reward for any given context vector. Coupled
with the results of Section~\ref{app:pf_bg:diversity}, it follows that batch history generated by a \GreedyStyle algorithm can simulate rewards as long as the batch size is sufficiently large.

Let us recap the definition of reward simulation (Definition~\ref{def:simulation}).  Let $\Rew(\cdot)$ be a randomized function that takes a context $x$ and outputs an independent random sample from $\mathcal{N}(\theta\tran x, 1)$. In other words, this is the realized reward for an action with context vector $x$.

\begin{definition}
Consider batch $B$ in the execution of a \GreedyStyle algorithm. Batch history $h_B$ can simulate $\Rew()$ up to radius $R>0$ if there exists a function
    $g: \{\text{context vectors}\}\times \{ \text{batch histories $h_B$}\} \to \R$
such that $g(x,h_B)$ is identically distributed to $\Rew(x)$ conditional on the batch context matrix, for all $\theta$ and all context vectors $x\in \R^d$ with $\|x\|_2\leq R$.
\label{def:simulation-app}
\end{definition}

Note that we do not require the function $g$ to be efficiently computable. We do not require algorithms to compute $g$; a mere existence of such function suffices for our analysis.

The result in this subsection does not rely on the ``greedy" property. Instead, it applies to all ``batch-style" algorithms, defined as follows: time is divided in batches of $Y$ consecutive rounds each, and the action at each round $t$ only depends on the history up to the previous batch. The data diversity condition is formalized as $\{\lambda_{\min}(Z_B) \ge R^2 \}$; recall that it is a high-probability event, in a precise sense defined in Lemma~\ref{lem:min_ev_bg}. The result is stated as follows:

\begin{lemma}
Fix a batch-style algorithm and any batch $B$ in the execution of this algorithm.
Assume the batch covariance matrix $Z_B$ satisfies $\lambda_{\min}(Z_B) \ge R^2$. Then batch history $h_B$ can simulate $\Rew$ up to radius $R$.
  \label{lem:lin_sim}
\end{lemma}

\begin{proof}
Let us construct a suitable function $g$ for Definition~\ref{def:simulation-app}. Fix a context vector $x\in \R^d$ with $\|x\|_2\leq R$. Let $r_B$ be the vector of realized rewards in batch $B$, \ie
    $r_B = (r_t: \text{rounds $t$ in $B$})\in \R^Y$. Define
\begin{align}\label{eq:lem:lin_sim:defn-g}
     g(x, h_B) = w_B\tran\, r_B + \mc N\left(0,1-\|w_B\|_2^2\right),
     \text{where $w_B = X_B\, Z_B^{-1}\, x\in \R^Y$}.
\end{align}

Recall that the variance of the reward noise is $1$. (We can also handle a more general version in which the variance of the reward noise is $\sigma^2$. Then the noise variance in \eqref{eq:lem:lin_sim:defn-g} should be $\sigma^2\,(1-\|w_B\|_2^2)$, with essentially no modifications throughout the rest of the proof.)

Note that $w_B$ is well-defined: indeed, $Z_B$ is invertible since
  $\lambda_{\min}(Z_B) \ge R^2>0$.
In the rest of the proof we show that $g$ is as needed for Definition~\ref{def:simulation-app}.

  First, we will show that for any $x \in \R^d$ such that $\|x\|_2 \le R$, the
  weights $\weights \in \R^t$ as defined above satisfy $X_B\tran \weights = x$
  and $\|\weights\|_2 \le 1$.
  Then, we'll show that if each $r_\tau \sim \mc N(\theta \tran x_\tau, 1)$,
  then $\vrt\tran \weights + \mc N(0, 1 - \|\weights\|_2^2)
  \sim \mc N(\theta\tran x, 1)$.

  Trivially, we have
  \[
    X_B\tran \weights = X_B\tran X_B (X_B\tran X_B)^{-1} x = x
  \]
  as desired. We must now show that $\|\weights\|^2_2 \le 1$. Note that
  \[
    \|\weights\|_2^2 = \weights\tran \weights = \weights\tran X_B Z_B^{-1} x =
    x\tran Z_B^{-1} x = \|x\|_{Z_B^{-1}}^2
  \]
  where $\|v\|_M^2$ simply denotes $v\tran M v$. Thus, it is sufficient to show
  that $\|x\|_{Z_B^{-1}}^2 \le 1$. Since
  $\|x\|_2 \le R$ and $\lambda_{\min}\p{Z_B} \ge R^2$, we have by
  Lemma~\ref{lem:norm_eigen}
  \[
    Z_B \succeq R^2 I \succeq xx\tran.
  \]
  By Lemma~\ref{lem:conj_succ}, we have
  \[
    I \succeq Z_B^{-1/2} xx\tran Z_B^{-1/2}.
  \]
  Let $z = Z_B^{-1/2} x$, so $I \succeq zz\tran$. Again by
  Lemma~\ref{lem:norm_eigen}, $\lambda_{\max}(zz\tran) = z\tran z$. This means
  that
  \[
    1 \ge z\tran z = (Z_B^{-1/2}x)\tran Z_B^{-1/2} x = x\tran Z_B^{-1} x =
    \|x\|_{Z_B^{-1}}^2 = \|\weights\|_2^2
  \]
  as desired.
  Finally, observe that
  \[
    \vrt\tran \weights = (X_B\theta + \eta)\tran \weights = \theta\tran X_B\tran
    \weights + \eta \tran \weights = \theta\tran x + \eta\tran \weights
  \]
  where $\eta \sim \mc N(0, I)$ is the noise vector. Notice that
  $\eta\tran \weights \sim \mc N(0, \|\weights\|_2)$, and therefore,
  $\eta\tran \weights + \mc N(0, 1-\|\weights\|_2^2) \sim \mc N(0,
  1)$. Putting this all together, we have
  \[
    \vrt\tran \weights + \mc N(0, 1-\|\weights\|_2^2) \sim \mc
    N(\theta\tran x, 1)
  \]
  and therefore $D$ can simulate $E$ for any $x$ up to radius $R$.
\end{proof}

\subsection{Regret Bounds for \BayesGreedy}
\label{sec:bg-proofs-bg}

We apply the tools from Sections~\ref{app:pf_bg:diversity} and~\ref{app:pf_bg:simulation} to derive regret bounds for \bg. On a high level, we prove that the history collected by \bg suffices to simulate a ``slowed-down" run of any other algorithm $\ALG_0$. Therefore, when it comes to choosing the next action, \bg has at least as much information as $\ALG_0$, so its Bayesian-greedy choice cannot be worse than the choice made by $\ALG_0$.

Our analysis extends to a more general scenario which is useful for the analysis of \fg. We formulate and prove our results for this scenario directly. We consider an extended bandit model which separates data collection and reward collection. Each round $t$ proceeds as follows: the algorithm observes available actions and the context vectors for these actions, then it chooses \emph{two} actions, $a_t$ and $a'_t$, and observes the reward for the former but not the latter. We refer to $a'_t$ as the ``prediction" at round $t$. We will refer to an algorithm in this model as a bandit algorithm (which chooses actions $a_t$) with ``prediction rule" that chooses the predictions $a'_t$. More specifically, we will be interested in an arbitrary \GreedyStyle algorithm with prediction rule given by \bg, as per
\eqref{eq:BG-est-defn} on \pageref{eq:BG-est-defn}. We assume this prediction rule henceforth. We are interested in \emph{prediction regret}: a version of regret \eqref{eq:regret-def} if actions $a_t$ are replaced with predictions $a'_t$:
\begin{align}\label{eq:pred-regret-def}
\PReg(T) = \textstyle
    \sum_{t=1}^T \theta\tran x_t^* -
\theta\tran x_{a'_t, t}
\end{align}
 where $x^*_{t}$ is the context vector of the best action at round $t$, as in \eqref{eq:regret-def}.
More precisely, we are interested in \emph{Bayesian prediction regret}, the expectation of \eqref{eq:pred-regret-def} over everything: the context vectors, the rewards, the algorithm's random seed, and and the prior over $\theta$.

We use essentially the same analysis to derive implications on group externalities.
For this purpose, we consider a further generalization in which regret is restricted to rounds that correspond to a particular population.
Formally, let $\rounds \subseteq \mathbb{N}$ be a randomly chosen subset of the rounds where
$\Pr[t \in \rounds]$ is a constant and rounds are chosen to be in $\rounds$ independently of
one another. We allow for the possibility that the underlying context
distribution differs for rounds in $\rounds$ compared to rounds in $[T] \backslash \mc
T$. More precisely, we allow the event $\{t\in \mc T\}$ be correlated with the context tuple at round $t$. Similar to the definition of minority regret, we define \emph{$\rounds$-restricted regret} (resp., prediction regret)
in $T$ rounds to be the portion of regret (resp., prediction regret) that corresponds to $\rounds$-rounds:
\begin{align}
  \rReg(T)
    &= \textstyle
    \sum_{t\leq T,\;t\in \rounds} \theta\tran x_t^* - \theta\tran x_{a_t, t}.
        \label{eq:restricted-regret-def} \\
  \rPReg(T) &= \textstyle
    \sum_{t\leq T,\;t\in \rounds} \theta\tran x_t^* -
\theta\tran x_{a'_t, t}.
    \label{eq:restricted-regret-def-pred}
\end{align}
$\rounds$-restricted \emph{Bayesian} (prediction) regret is defined as
an expectation over everything.

Thus, the main theorem of this subsection is formulated as follows:

\begin{theorem}
  Consider perturbed context generation. Let $\ALG$ be an arbitrary \GreedyStyle
  algorithm whose batch size is at least $Y_0$ from \eqref{eq:lem:min_ev_bg-Y}.
Fix any bandit algorithm $\ALG_0$, and let
    $\rReg_0(T)$
be the $\rounds$-restricted regret of this algorithm on a particular problem
instance $\mc I$. Then on the same instance, $\ALG$ has $\rounds$-restricted Bayesian prediction regret
\begin{align}\label{eq:thm:bg}
  \E{\rPReg(T)} \leq Y \cdot \E{\rReg_0(T/Y)} + \tilde O(1/T).
\end{align}
\label{thm:bg}
\end{theorem}

\xhdr{Proof sketch.} We use a $t$-round history of \ALG to simulate a $(t/Y)$-round history of $\ALG_0$. More specifically, we use each batch in the history of \ALG to simulate one round of $\ALG_0$. We prove that the simulated history of $\ALG_0$ has exactly the same distribution as the actual history, for any $\theta$. Since $\ALG$ predicts the Bayesian-optimal action  given the history (up to the previous batch), this action is at least as good (in expectation over the prior) as the one chosen by $\ALG_0$ after $t/Y$ rounds. The detailed proof is deferred to Section~\ref{sec:thm-bg-pf}.

\xhdr{Implications.} As a corollary of this theorem, we obtain regret bounds for \bg in Theorem~\ref{thm:main-greedy} and Theorem~\ref{thm:main-worst-case}. We take $\rounds$ to be the set of all rounds, \ie $\Pr[t \in \rounds] = 1$, and $\ALG$ to be \bg. For Theorem~\ref{thm:main-worst-case}(b), we take $\ALG_0$ to be LinUCB. Thus:

\begin{corollary}
In the setting of Theorem~\ref{thm:bg}, \bg has Bayesian regret at most
    $Y \cdot \E{R_0(T/Y)} + \tilde O(1/T)$
on problem instance $\mc I$. Further, under the assumptions of Theorem~\ref{thm:main-worst-case}, \bg has Bayesian regret at most
    $\tilde O(d^2\,K^{2/3}\;T^{1/3}/\rho^2)$
on all instances.
\end{corollary}

We also obtain a similar regret bound on the Bayesian prediction regret of \fg, which is essential for Section~\ref{sec:bg-proofs-fg}.

\begin{corollary}\label{cor:thm-bg-fg}
In the setting of Theorem~\ref{thm:bg}, \fg has Bayesian prediction regret \eqref{eq:thm:bg}.
\end{corollary}

To derive Theorem~\ref{thm:main-greedy-externalities} for \bg, we take $\rounds$ to be the set of all minority rounds, and apply Theorem~\ref{thm:bg} twice: first when $\ALG_0$ is run over the minority rounds only (and can behave arbitrarily on the rest), and then when $\ALG_0$ is run over full population.

\subsubsection{Proof of Theorem~\ref{thm:bg}}
\label{sec:thm-bg-pf}

We condition on the event that all perturbations are bounded by $\hat{R}$, more precisely, on the event
\begin{align}\label{eq:thm:bg-pf-E1}
\mE_1 = \left\{ \|\eps_{a,t}\|_\infty \leq \hat R:\;
    \text{ for all arms $a$ and all rounds $t$ } \right\}.
\end{align}
Recall that $\mE_1$ is a high-probability event, by \eqref{eq:bg-proofs-highprob-R-hat}.
We also condition on the event
\[
  \mE_2 = \left\{ \lambda_{\min}(Z_B) \ge R^2
  : \; \text{for each batch $B$},\right\}
\]
where $Z_B$ is the batch covariance matrix, as usual. Conditioned on $\mE_1$, this too is a high-probability event by
Lemma~\ref{lem:min_ev_bg} plugging in $\delta/T$ and taking a union bound over
all batches.

We will prove that \ALG satisfies
\begin{align}\label{eq:bg-cond}
\E{\rPReg(T) \given \mE_1, \mE_2}
    \leq Y \cdot \E{\rReg_0(\Cel{T/Y}) \given \mE_1, \mE_2},
\end{align}
where the expectation is taken over everything: the context vectors, the rewards, the algorithm's random seed, and the prior over $\theta$. Then we take care of the ``failure event" 
    $\overline{\mE_1 \cap \mE_2}$.

\xhdr{History simulation.}
Before we prove \eqref{eq:bg-cond}, let us argue about using the history of $\ALG$ to simulate a (shorter) run of $\ALG_0$. Fix round $t$. We use a $t$-round history of $\ALG$ to simulate a $\ty$-round run of $\ALG_0$, where $Y$ is the batch size in $\ALG$. Stating this formally requires some notation.  Let $A_t$ be the set of actions available in round $t$, and let
    $\con_t = (x_{a,t}:\, a\in A_t)$
be the corresponding tuple of contexts. Let $\CON$ be the set of all possible context tuples, more precisely, the set of all finite subsets of $\R^d$. Let $h_t$ and $h^0_t$ denote, resp., the $t$-round history of $\ALG$ and $\ALG_0$. Let $\mH_t$  denote the set of all possible $t$-round histories. Note that $h_t$ and $h^0_t$ are random variables which take values on $\mH_t$.  We want to use history $h_t$ to simulate history $h^0_\ty$. Thus, the simulation result is stated as follows:

\begin{lemma}\label{lm:bg-simulation}
Fix round $t$ and let $\sigma = (\con_1 \LDOTS \con_\ty)$ be the sequence of context arrivals up to and including round $\ty$. Then there exists a ``simulation function"
    \[ \simF = \simF_t: \mH_t\times \CON_{\ty} \to \mH_{\ty} \]
such that the simulated history $\simF(h_t,\sigma)$ is distributed identically to $h^0_{\ty}$, conditional on sequence $\sigma$, latent vector $\theta$, and events $\mE_1,\mE_2$.
\end{lemma}

\begin{proof}
Throughout this proof, condition on events $\mE_1$ and $\mE_2$. Generically,
$\simF(h_t,\sigma)$ outputs a sequence of pairs
    $\{(x_\tau, r_\tau)\}_{\tau=1}^{\lfloor t/Y \rfloor}$,
where $x_\tau$ is a context vector and $r_\tau$ is a simulated reward for this
context vector. We define $\simF(h_t,\sigma)$ by induction on $\tau$ with base case $\tau=0$. Throughout, we maintain a run of algorithm $\ALG_0$. For each step $\tau\geq 1$, suppose $\ALG_0$ is simulated up to round $\tau-1$, and the corresponding history is recorded as
    $((x_1,r_1) \LDOTS (x_{\tau-1},r_{\tau-1}))$.
Simulate the next round in the execution of $\ALG_0$ by presenting it with the action set $A_\tau$ and the corresponding context tuple $\con_\tau$. Let $x_\tau$ be the context vector chosen by $\ALG_0$. The corresponding reward $r_\tau$ is constructed using the $\tau$-th batch in $h_t$, denote it with $B$. By Lemmas~\ref{lem:min_ev_bg} and~\ref{lem:lin_sim}, the batch history
$h_B$ can simulate a single reward, in the sense of
Definition~\ref{def:simulation-app}. In particular, there exists a function
$g(x,h_B)$ with the required properties (recall that it is explicitly defined in
\eqref{eq:lem:lin_sim:defn-g}). Thus, we define $r_\tau = g(x_\tau,h_B)$, and
return $r_\tau$ as a reward to $\ALG_0$. This completes the construction of
$\simF(h_t,\sigma)$. The distribution property of $\simF(h_t,\sigma)$ is immediate from the construction.
\end{proof}

\begin{proofof}[of Equation~\eqref{eq:bg-cond}]
We argue for each batch separately, and then aggregate over all batches in the very end. Fix batch $B$, and let $t_0 = t_0(B)$ be the last round in this batch. Let $\tau = 1+t_0/Y$, and consider the context vector $x^0_\tau$ chosen by $\ALG_0$ in round $\tau$. This context vector is a randomized function $f$ of the current context tuple $\con_\tau$ and the history $h^0_{\tau-1}$:
    \[ x^0_\tau = f(\con_\tau; h^0_{\tau-1}).\]
By Lemma~\ref{lm:bg-simulation}, letting
     $\sigma = (\con_1 \LDOTS \con_\ty)$,
it holds that
\begin{align}\label{eq:bg-proof-tau}
 \E{ x^0_\tau \cdot \theta \given \sigma,\theta,\mE_1,\mE_2}
    = \E{ f(\con_\tau;\,\simF(h_{t_0},\sigma)) \cdot \theta \given \sigma,\theta,\mE_1,\mE_2}
\end{align}

Let $t$ be some round in the next batch after $B$, and let
    $x'_t = x_{a'_t,t}$,
be the context vector predicted by $\ALG$ in round $t$. Recall that $x'_t$ is a Bayesian-greedy choice from the context tuple $\con_t$, based on history $h_{t_0}$.
Observe that the Bayesian-greedy action choice from a given context tuple based on history $h_{t_0}$ cannot be worse, in terms of the Bayesian-expected reward, than any other choice from the same context tuple and based on the same history. Using \eqref{eq:bg-proof-tau}, we obtain:
\begin{align}\label{eq:bg-proof-MII-cond}
 \E{ x'_t \cdot \theta \given \con_t = \con,\mE_1,\mE_2  }
    \geq \E{ x^0_\tau \cdot \theta  \given \con_\tau = \con,\mE_1,\mE_2},
 \end{align}
for any given context tuple $\con\in\CON$ that has a non-zero arrival probability given $\mE_1 \cap\mE_2$. 

Given $\con_t = \con$, the event $t\in\rounds$ is independent of everything else. Likewise, given $\con_\tau = \con$, the event $\tau\in\rounds$ is independent of everything else. It follows that
\begin{align}\label{eq:bg-proof-MII-cond-cond}
 \E{ x'_t \cdot \theta \given \con_t = \con,t\in \rounds, \mE_1,\mE_2  }
    \geq \E{ x^0_\tau \cdot \theta  \given \con_\tau = \con, \tau\in\rounds,\mE_1,\mE_2},
 \end{align}
for any given context tuple $\con\in\CON$ that has a non-zero arrival probability given $\mE_1 \cap\mE_2$.

Observe that $\con_t$ and $\con_\tau$ have the same distribution, even conditioned on event $\mE_1 \cap\mE_2$. (This is because the definitions of $\mE_1$ and $\mE_2$ treat all rounds in the same batch in exactly the same way.)
Therefore, we can integrate \eqref{eq:bg-proof-MII-cond-cond} over the context tuples $\con$:
\begin{align}\label{eq:bg-proof-MII}
 \E{ x'_t \cdot \theta \given t\in\rounds,\mE_1,\mE_2  }
    \geq \E{ x^0_\tau \cdot \theta  \given \tau\in\rounds,\mE_1,\mE_2},
 \end{align}
Now, let us sum up \eqref{eq:bg-proof-MII} over all rounds $t$ in the next batch after $B$, denote it $\term{next}(B)$.
\begin{align}\label{eq:bg-proof-MII-B}
 \sum_{t\in \term{next}(B)} \E{ x'_t \cdot \theta \given t\in\rounds,\mE_1,\mE_2  }
    \geq Y\cdot \E{ x^0_\tau \cdot \theta  \given \tau\in\rounds,\mE_1,\mE_2}.
 \end{align}
Note that the right-hand side of \eqref{eq:bg-proof-MII} stays the same for all $t$, hence the factor of $Y$ on the right-hand side of \eqref{eq:bg-proof-MII-B}. This completes our analysis of a single batch $B$. 

We obtain~\eqref{eq:bg-cond} by over all batches $B$. Here it is essential that the expectation
    $\E{\ind{t\in\rounds}\;\theta\tran x_t^*}$
does not depend on round $t$, and therefore the ``regret benchmark" $\theta\tran x_t^*$ cancels out from~\eqref{eq:bg-cond}. In particular, it is essential that the context tuples $\con_t$ are identically distributed across rounds.
\end{proofof}

\begin{proofof}[of Theorem~\ref{thm:bg} given Equation~\eqref{eq:bg-cond}]
We must take care of the low-probability failure events $\overline{\mE}_1$ and
$\overline{\mE}_2$.
Specifically, we need to upper-bound the expression
  \[
    \Exp_{\theta \sim P} \b{\bpreg{T} \given \overline{\mc E}_1 \cup \overline{\mc
    E}_2} \cdot \Pr[\overline{\mc E}_1 \cup \overline{\mc E}_2].
  \]
  For ease of exposition, we focus on the special case  $\Pr\b{t \in \rounds} = 1$; the general case is treated similarly. We know that $\Pr[\overline{\mc E}_1 \cup \overline{\mc E}_2] \le \delta +
  \delta_R$.
  Lemma~\ref{lem:exp_reg_ub_er} with $\ell = \hat R$
  gives us that the instantaneous regret of every round is at most
  \begin{align*}
    2\Exp_{\theta \sim (\prior \given h_{t-1})} & \b{\|\theta\|_2\p{1 + \rho(2 +
      \sqrt{2 \log K}) + \hat R}} \\
    &\le 2\b{\p{\|\pmt\|_2 + \sqrt{d\lambda_{\max}(\pvt)}}\p{1 + \rho(2 +
      \sqrt{2 \log K}) + \hat R}}
  \end{align*}
  by Lemma~\ref{lem:gaus_norm}. Letting $\delta = \delta_R = \frac{1}{T^2}$, we
  verify that our definition of $Y$ means that Lemma~\ref{lem:min_ev_bg} indeed
  holds with probability at least $1-T^{-2}$.
  Using~\eqref{eq:bg-cond}, the Bayesian prediction regret of $\ALG$ is
  \begin{align*}
    \Exp_{\theta \sim \prior} &\b{\bpreg{T}} \\
    &\le Y \Exp_{\theta \sim \prior} \b{\basereg{\tfrac{T}{Y}}}
    + 2\,T(\delta + \delta_R)\b{\p{\|\pmt\|_2 + \sqrt{d\lambda_{\max}(\pvt)}}\p{1
      + \rho(2 + \sqrt{2 \log K}) + \hat R}} \\
    &\le Y \Exp_{\theta \sim \prior} \b{\basereg{\tfrac{T}{Y}}} + \tilde
    O\p{\tfrac{1}{T}}. 
  \end{align*}
This completes the proof of Theorem~\ref{thm:bg}.
\end{proofof}

\subsection{Regret Bounds for \FreqGreedy}
\label{sec:bg-proofs-fg}

To analyze \fg, we show that its Bayesian regret is not too different from its Bayesian prediction regret, and use Corollary~\ref{cor:thm-bg-fg} to bound the latter. As in the previous subsection, we state this result in more generality for the sake of group externality implications: we consider $\rounds$-restricted (prediction) regret, exactly as before.

\begin{theorem}
  Assuming perturbed context generation, \fg satisfies
  \[  \left|\; \E{\rReg(T) - \rPReg(T)} \; \right| \leq
    \tilde O\p{\frac{\sqrt{d}}{\rho^2}} \p{\sqrt{\lambda_{\max}(\pvt)} +
    \frac{1}{\sqrt{\lambda_{\min}(\pvt)}}},
  \]
  where $\pvt$ is the covariance matrix of the prior and $\rho$ is the perturbation size.
  \label{thm:bg_fg}
\end{theorem}

Taking $\rounds$ to be the set of all contexts, and using Corollary~\ref{cor:thm-bg-fg}, we obtain Bayesian regret bounds for \fg in Theorem~\ref{thm:main-greedy} and Theorem~\ref{thm:main-worst-case}.
To derive Theorem~\ref{thm:main-greedy-externalities} for \fg, we take $\rounds$ to be the set of all minority rounds.

The remainder of this section is dedicated to proving Theorem~\ref{thm:bg_fg}. On a high level, the idea is as follows. As in the proof of Theorem~\ref{thm:bg}, we condition on the high-probability event \eqref{eq:thm:bg-pf-E1} that perturbations are bounded. Specifically, we prove that
\begin{align}\label{eq:thm:bg_fg-cond}
  \left|\; \E{\rReg(T) - \rPReg(T) \given \mE_1} \; \right| \leq
    \tilde O\p{\frac{\sqrt{d}}{\rho^2}} \p{\sqrt{\lambda_{\max}(\pvt)} +
    \frac{1}{\sqrt{\lambda_{\min}(\pvt)}}}.
\end{align}
To prove this statement, we fix round $t$ and compare the action $a_t$ taken by \fg and the predicted action $a'_t$. We observe that the difference in rewards between these two actions can be upper-bounded in terms of $\bmt-\fmt$,
the difference in the $\theta$ estimates with and without knowledge of the prior. (Recall \eqref{eq:BG-est-defn} and \eqref{eq:FG-est-defn} for definitions.)
Specifically, we show that
\begin{equation}
\label{eq:inst_bound_diff}
  \E{
    \theta\tran (x_{a_t, t} - x_{a'_t, t}) \given \mE_1 }
    \le 2R\Exp_{\theta \sim \prior}\b{\|\bmt - \fmt\|_2}.
\end{equation}
The crux of the proof is to show that the difference $\|\bmt - \fmt\|_2$ is small, namely
\begin{equation}
  \label{eq:norm_bound_1_t}
  \E{\|\bmt-\fmt\|_2 \given \mE_1} = \tilde O(1/t),
\end{equation}
ignoring other parameters. Thus, summing over all rounds, we get
\[ \E{\rReg(T) - \rPReg(T) \given
\mE_1} \le O(\log T) = \tilde O(1). \]

\xhdr{Proof of Eq.~\eqref{eq:thm:bg_fg-cond}.}
Let $\regi{t}$ and $\bpregi{t}$ be, resp., instantaneous regret and instantaneous prediction regret at time $t$. Then
  \begin{equation}
    \Exp_{\theta \sim \prior}\b{\rReg(T) - \rPReg(T)}
    = \sum_{t\in \rounds} \Exp_{\theta \sim
    \prior} \b{\regi{t} - \bpregi{t}}.
    \label{eq:reg_time}
  \end{equation}
  Thus, it suffices to bound the differences in instantaneous regret.

  Recall that at time $t$, the chosen action for \fg\ and the predicted action are, resp.,
  \begin{align*}
    \af &= \argmax_{a \in A} x_{a,t}\tran \fmt \\
    \ab &= \argmax_{a \in A} x_{a,t}\tran \bmt.
  \end{align*}
Letting $t_0 - 1 = \lfloor t/Y \rfloor$ be the last round in the previous batch,
we can formulate $\fmt$ and $\bmt$ as
\begin{align*}
    \fmt &= (\Zto)^{-1} \Xto\tran \vrto \\
    \bmt &= (\Zto + \pvt^{-1})^{-1} (\Xto\tran \vrto + \pvt^{-1} \pmt).
\end{align*}

  Therefore, we have
  \[
    \Exp_{\theta \sim \prior \given h_{t-1}}\b{\regi{t} - \bpregi{t}} = \Exp_{\theta \sim
      \prior \given h_{t-1}} \b{(x_{\ab,t} - x_{\af,t})\tran \bmt} =
      (x_{\ab,t} - x_{\af,t})\tran \bmt,
  \]
  since the mean of the posterior distribution is exactly $\bmt$, and $\bmt$ is
  deterministic given $h_{t-1}$. Taking expectation over $h_{t-1}$, we have
  \[
    \Exp_{\theta \sim \prior}\b{\regi{t} - \bpregi{t}} = \Exp_{\theta \sim \prior}
    \b{(x_{\ab,t} - x_{\af,t})\tran \bmt}.
  \]
  For any fixed $\bmt$ and $\fmt$, since \fg\ chose $\af$ over $\ab$, it must be
  the case that
  \begin{equation}
    x_{\af,t}\tran \fmt \ge x_{\ab,t}\tran \fmt.
    \label{eq:freq_choice}
  \end{equation}
  Therefore,
  \begin{align*}
    (x_{\ab,t} - x_{\af,t})\tran \bmt &= (x_{\ab,t} - x_{\af,t})\tran \fmt +
    (x_{\ab,t} - x_{\af,t})\tran (\bmt - \fmt) \\
    &\le (x_{\ab,t} - x_{\af,t})\tran (\bmt - \fmt)
    \tag{By~\eqref{eq:freq_choice}} \\
    &\le (\|x_{\ab,t}\|_2 + \|x_{\af,t}\|_2)\|\bmt - \fmt\|_2 \\
    &\le 2R\|\bmt - \fmt\|_2
  \end{align*}
Eq.~\eqref{eq:inst_bound_diff} follows.

The crux is to prove \eqref{eq:norm_bound_1_t}: to bound the expected distance between the Frequentist and Bayesian estimates for $\theta$. By expanding
  their definitions, we have
  \begin{align*}
    \bmt - \fmt
    &= (\Zto + \pvt^{-1})^{-1} (\Xto\tran \vrto + \pvt^{-1}
    \pmt) - \Zto^{-1} \Xto\tran \vrto \\
    &= (\Zto + \pvt^{-1})^{-1} \b{\Xto\tran \vrto + \pvt^{-1} \pmt -
    (\Zto + \pvt^{-1})\Zto^{-1} \Xto\tran \vrto} \\
    &= (\Zto + \pvt^{-1})^{-1} \b{\Xto\tran \vrto + \pvt^{-1} \pmt -
    \Xto\tran \vrto  - \pvt^{-1}\Zto^{-1} \Xto\tran \vrto} \\
    &= (\Zto + \pvt^{-1})^{-1} \b{\pvt^{-1} \pmt -
    \pvt^{-1}\Zto^{-1} \Xto\tran \vrto} \\
    &= (\Zto + \pvt^{-1})^{-1} \pvt^{-1} \p{\pmt - \fmt}.
  \end{align*}
  Next, note that
  \begin{align*}
    \|(\Zto + \pvt^{-1})^{-1} \pvt^{-1} (\pmt - \fmt)\|_2
    &\le \|(\Zto + \pvt^{-1})^{-1}\|_2 ~ \|\pvt^{-1} (\pmt - \fmt)\|_2 \\
    &\le \|(\Zto + \pvt)^{-1}\|_2 ~ \p{\|\pvt^{-1}(\pmt - \theta)\|_2
    + \|\pvt^{-1}\|_2 ~ \|\theta - \fmt\|_2}.
  \end{align*}
  By Lemma~\ref{lem:min_ev_sum}, $\lambda_{\min}\p{\Zto + \pvt} \ge
  \lambda_{\min}\p{\Zto}$. Therefore,
  \[
    \|(\Zto + \pvt)^{-1}\|_2 \le \frac{1}{\lambda_{\min}\p{\Zto}},
  \]
  giving us
  \begin{align*}
    \|\bmt - \fmt\|_2
    &\le \frac{\|\pvt^{-1}(\pmt - \theta)\|_2 + \|\pvt^{-1}\|_2
    ~ \|\theta - \fmt\|_2}{\lambda_{\min}(\Zto)} \\
    &\le \frac{\|\pvt^{-1/2}\|_2 \|\pvt^{-1/2}(\pmt - \theta)\|_2 + \|\pvt^{-1/2}\|_2
    ~ \|\theta - \fmt\|_2}{\lambda_{\min}(\Zto)} \\
    &= \frac{\p{\|\pvt^{-1/2}(\pmt - \theta)\|_2 + \sqrt{\lambda_{\min}(\pvt)}
    \|\theta - \fmt\|_2}}{\sqrt{\lambda_{\min}(\pvt)} \lambda_{\min}(\Zto)}.
  \end{align*}

  Next, recall that for
  \[ t_0-1 \ge t_{\min}(\delta) := 160 \tfrac{R^2}{\rho^2} \log \tfrac{2d}{\delta} \cdot \log T \]
 the following bounds hold, each with probability at least $1-\delta$:
  \begin{align*}
    \frac{1}{\lambda_{\min}\p{\Zto}} &\le \frac{32 \log T}{\rho^2
    (t_0-1)}
    \tag{Lemma~\ref{lem:fg_big_cov}} \\
    \|\theta - \fmt\|_2 &\le \frac{\sqrt{2dR (t_0-1)
    \log(d/\delta)}}{\lambda_{\min}(\Zto)} \tag{Lemma~\ref{lem:fmt_close}}
  \end{align*}
 Therefore, fixing $t_0 \geq 1+t_{\min}(\delta/2)$, with probability at least $1-\delta$ we have
  \begin{equation}
    \|\bmt - \fmt\|_2
    \le \frac{32 \log T}{\rho^2 (t_0-1) \sqrt{\lambda_{\min}(\pvt)}}
    \p{\|\pvt^{-1/2}(\pmt - \theta)\|_2 + \frac{64\sqrt{dR
    \log(2d/\delta)} \cdot \log T}{\rho^2 \sqrt{t_0-1}}}.
    \label{eq:fg_bg1}
  \end{equation}
  Note that the high-probability events we need are deterministic given
  $h_{t_0-1}$, and therefore are independent of the perturbations at time $t$.
  This means that Lemma~\ref{lem:exp_reg_ub_er} applies, with $\ell = 0$: conditioned on
  any $h_{t_0-1}$, the expected regret for round $t$ is upper-bounded by
  $2\|\theta\|_2 (1 + \rho(1+\sqrt{2\log K}))$. In particular, this holds for any
  $h_{t_0-1}$ not satisfying the high probability events from
  Lemmas~\ref{lem:fg_big_cov} and~\ref{lem:fmt_close}. Therefore, for all $t \ge
  t_{\min}(\delta)$,
  \begin{align*}
&~~    \Exp_{\theta \sim \prior} \b{\|\bmt - \fmt\|_2}\\
    &\le \Exp_{\theta \sim \prior} \Bigg[(1-\delta) \frac{32 \log T}{\rho^2
      (t_0-1) \sqrt{\lambda_{\min}(\pvt)}} \p{\|\pvt^{-1/2}(\pmt - \theta)\|_2 +
      \frac{64\sqrt{dR \log(2d/\delta)} \cdot \log T}{\rho^2 \sqrt{t_0-1}}} \\
    &\qquad\qquad+ \delta \cdot 2\|\theta\|_2 (1 + \rho(2+\sqrt{2\log K})) \Bigg] \\
    &\le \frac{32 \log T}{\rho^2 (t_0-1) \sqrt{\lambda_{\min}(\pvt)}} \p{\Exp_{\theta \sim
    \prior}\b{\|\pvt^{-1/2} (\pmt - \theta)\|_2} + \frac{64 \sqrt{dR
    \log(2d/\delta)} \cdot \log T}{\rho^2 \sqrt{t_0-1}}} \\
    &\qquad+ \delta \cdot 2(\|\pmt\|_2 + \Exp_{\theta \sim \prior}\b{\|\pmt -
    \theta\|_2}) (1 + \rho(2+\sqrt{2\log K})).
  \end{align*}
  Because $\theta \sim \mc N(\pmt, \pvt)$, we have $\pvt^{-1/2}
  (\pmt - \theta) \sim \mc N(0, I)$. By Lemma~\ref{lem:gaus_norm},
  \[
    \Exp_{\theta \sim \prior} \b{\|\pvt^{-1/2} (\pmt - \theta)\|_2}
    \le \sqrt{d}
    \quad\text{and}\quad
    \Exp_{\theta \sim \prior}\b{\|\pmt - \theta\|_2} \le \sqrt{d
      \lambda_{\max}(\pvt)}.
  \]
  This means
  \begin{align*}
    \Exp_{\theta \sim \prior} \b{\|\bmt - \fmt\|_2}
    &\le \frac{32 \sqrt{d} \log T}{\rho^2 (t_0-1) \sqrt{\lambda_{\min}(\pvt)}} \p{1 +
      \frac{64 \sqrt{R \log(2d/\delta)} \cdot \log T}{\rho^2 \sqrt{t_0-1}}} \\
    &+ \delta \cdot 2(\|\pmt\|_2 + \sqrt{d \lambda_{\max}(\pvt)}) (1 +
    \rho(2+\sqrt{2\log K})).
  \end{align*}
  Since $t_0 = \Omega(t)$, for sufficiently small $\delta$, this
  proves~\eqref{eq:norm_bound_1_t}.

  We need to do a careful computation to complete the proof of Eq.~\eqref{eq:thm:bg_fg-cond}.  We know from~\eqref{eq:inst_bound_diff} that
  \begin{align*}
    \Exp_{\theta \sim \prior}\b{\rReg(T) - \rPReg(T)}
    &\le \sum_{t=1}^T 2R\Exp_{\theta \sim \prior} \b{\|\bmt - \fmt\|_2}.
  \end{align*}
  Choosing $\delta = T^{-2}$, we find that
  \[
    \sum_{t=t_{\min}(T^{-2})}^T \delta \cdot 2(\|\pmt\|_2 + \sqrt{d
    \lambda_{\max}(\pvt)}) (1 + \rho(2+\sqrt{2\log K})) = \tilde O(1),
  \]
  so this term vanishes. Furthermore,
  \[
    \sum_{t=t_{\min}(T^{-2})}^T 2R\frac{32 \sqrt{d} \log T}{\rho^2 (t_0-1)
    \sqrt{\lambda_{\min}(\pvt)}} \p{1 + \frac{64\sqrt{R \log(2d/\delta)} \cdot
    \log T}{\rho^2 \sqrt{t_0-1}}} = \tilde
    O\p{\frac{R\sqrt{d}}{\rho^2\sqrt{\lambda_{\min}(\pvt)}}}
  \]
  since $t_0 \ge t - Y$, and $\sum_{t=1}^T 1/t = O(\log T)$.
  Using the fact that $R = \tilde O(1)$ (since by assumption $\rho \le
  d^{-1/2}$), this is simply
  \[
    \tilde O\p{\frac{\sqrt{d}}{\rho^2\sqrt{\lambda_{\min}(\pvt)}}}.
  \]
  Finally, we note that on the first $t_{\min}(T^{-2}) = \tilde O(1/\rho^2)$
  rounds, the regret bound from Lemma~\ref{lem:exp_reg_ub_er} with $\ell = 0$
  applies, so the total regret difference is at most
  \begin{align*}
    \Exp_{\theta \sim \prior}\b{\rReg(T) - \rPReg(T)}
    &\le \sum_{t=1}^{t_{\min}(T^{-2})}
    \Exp_{\theta \sim \prior}\b{\regi{t} - \bpregi{t}}
    + \sum_{t=t_{\min}(T^{-2})}^T 2R\Exp_{\theta \sim \prior} \b{\|\bmt - \fmt\|_2}, \\
    &\le t_{\min}(T^{-2}) \cdot 2(\|\pmt\|_2 + \sqrt{d
    \lambda_{\max}(\pvt)})(1 + \rho(2+\sqrt{2 \log K}))
    + \tilde O\p{\frac{\sqrt{d}}{\rho^2\sqrt{\lambda_{\min}(\pvt)}}} \\
    &= \tilde O\p{\frac{\sqrt{d \lambda_{\max}(\pvt)}}{\rho^2}}
    + \tilde O\p{\frac{\sqrt{d}}{\rho^2\sqrt{\lambda_{\min}(\pvt)}}},
  \end{align*}
which implies Eq.~\eqref{eq:thm:bg_fg-cond}.

\xhdr{Completing the proof of Theorem~\ref{thm:bg_fg} given ~\eqref{eq:thm:bg_fg-cond}.}
  By Theorem~\ref{thm:bg_fg}, this holds whenever all perturbations are bounded by
  $\hat R$, which happens with probability at least $1-\delta_R$. When the bound
  fail, the total regret is at most
  \begin{align*}
    2\b{\p{\|\pmt\|_2 + \sqrt{d\lambda_{\max}(\pvt)}}\p{1 + \rho(2 +
      \sqrt{2 \log K}) + \hat R}}
  \end{align*}
  by Lemma~\ref{lem:exp_reg_ub_er} (with $\ell = \hat R$)
  and Lemma~\ref{lem:gaus_norm}. Since $\delta_R = T^{-2}$, the contribution of
  regret when the high-probability bound fails is $\tilde O(1/T) \le \tilde
  O(1)$.

%% file: sections/lemmas.tex
Throughout the paper, we use a number of tools that are either known or easily follow from something that is known. We move these tools to a separate appendix so as not to interrupt the flow. We provide the proofs for the sake of completeness.

\subsection{(Sub)gaussians and Concentration}

We rely on several known facts about Gaussian and subgaussian random variables. A random variable $X$ is called $\sigma$-subgaussian, for some $\sigma>0$, if $E[e^{\sigma X^2}]<\infty$. This includes variance-$\sigma^2$ Gaussian random variables as a special case.

\begin{lemma}
  If $X \sim \mc N(0, \sigma^2)$, then for any $t \ge 0$,
  \[
    \E{X \given X \ge t} \le \begin{cases}
      2 \sigma & t \le \sigma \\
      t + \frac{\sigma^2}{t} & t > \sigma
    \end{cases}
  \]
  \label{lem:gaus_exp_bound}
\end{lemma}
\begin{proof}
  We begin with
  \begin{align}\label{eq:pf:lem:gaus_exp_bound}
    \E{X \given X \ge t} = \frac{\frac{1}{\sigma\sqrt{2\pi}}\int_t^\infty x
    \exp\p{x^2/(2\sigma^2)} \dx}{\Pr\b{X \ge t}}.
  \end{align}
$X$ can be represented as $X = \sigma Y$, where $Y$ is a standard normal random variable. Using a tail bound for the latter (from \citet{cook2009upper}),
\[
    \Pr\b{X \ge t} = \Pr\b{Y \ge \frac{t}{\sigma}} \ge
    \frac{1}{\sqrt{2\pi}} \frac{t/\sigma}{(t/\sigma)^2 + 1}
    \exp\p{-\frac{t^2}{2\sigma^2}}.
  \]
  The numerator in \eqref{eq:pf:lem:gaus_exp_bound} is
  \begin{align*}
    \frac{1}{\sigma\sqrt{2\pi}}\int_t^\infty x \exp\p{x^2/(2\sigma^2)} \dx
    &= -\frac{1}{\sigma \sqrt{2\pi}} \cdot \sigma^2 e^{-x^2/(2\sigma^2)}
    \bigg|_t^\infty \cdot e^{-t^2/(2\sigma^2)}
    = \frac{\sigma}{\sqrt{2\pi}} \exp\p{-\frac{t^2}{2\sigma^2}}.
  \end{align*}
  Combining, we have
  \begin{align*}
    \E{X \given X \ge t} &\le \frac{\frac{\sigma}{\sqrt{2\pi}}
    \exp\p{-\frac{t^2}{2\sigma^2}}}{\frac{1}{\sqrt{2\pi}}\frac{t/\sigma}{(t/\sigma)^2
    + 1} \exp\p{-\frac{t^2}{2\sigma^2}}}
    = \frac{\sigma^2 ((t/\sigma)^2 + 1)}{t}
    = t + \frac{\sigma^2}{t}
  \end{align*}
  For $t \le \sigma$, $\E{X \given X \ge t} \le \E{X \given X \ge \sigma} \le
  2\sigma$ by the above bound.
\end{proof}
\begin{lemma}
 Suppose $X \sim \mc N(0, \Sigma)$ is a Gaussian random vector with covariance matrix $\Sigma$. Then
  \[
    \E{\; \|X\|_2 \given \|X\|_2 > \alpha \;} 
        \le d\p{\alpha+ \frac{\lambda_{\max}(\Sigma)}{\alpha}} \quad\
    \text{for any $\alpha \ge 0$}.
  \]
  \label{lem:gaus_exp_norm_bound}
\end{lemma}
\begin{proof}
  Assume without loss of generality that $\Sigma$ is diagonal, since the norm is
  rotationally invariant. Observe that
    $\|X\|_2 \given \forall i ~ X_i > \alpha$
  stochastically dominates $\|X\|_2 \given \|X\|_2 > \alpha$.
  (Geometrically, the latter conditioning shifts the probability mass away from the origin.)
  Therefore,
  \begin{align*}
    \E{\; \|X\|_2 \given \|X\|_2 > \alpha \;}
    &\le \E{\; \|X\|_2 \given \forall i ~ X_i > \alpha \;} \\
    &= \textstyle  \E{\sum_{i=1}^d X_i \given \forall i ~ X_i > \alpha}
    \le \sum_{i=1}^d \p{t+ \frac{\lambda_i(\Sigma)}{\alpha}}
  \end{align*}
  by Lemma~\ref{lem:gaus_exp_bound}, where
    $\lambda_i(\Sigma) \leq \lambda_{\max}(\Sigma)$ is the $i$th
  eigenvalue of $\Sigma$.
\end{proof}

\begin{fact}
  If $X$ is a $\sigma$-subgaussian random variable, then
  \[
    \Pr[|X-\E{X}| > t] \le 2e^{-t^2/(2\sigma^2)}.
  \]
  \label{fact:subg_def}
\end{fact}

\begin{lemma}
  If $X_1, \dots, X_n$ are independent $\sigma$-subgaussian random variables, then
  \begin{align*}
    \Pr\b{\max_i |X_i-\E{X_i}| > \sigma\sqrt{2\log\frac{2n}{\delta}}} \le \delta.
  \end{align*}
  \label{lem:subg_union_bound}
\end{lemma}
\begin{proof}
  For any $X_i$, we know from Fact~\ref{fact:subg_def} that
  \[
    \Pr\b{|X_i - \E{X_i}| > \sigma \sqrt{2 \log \frac{2n}{\delta}}}
    \le 2\exp\p{-\frac{2\sigma^2 \log \frac{2n}{\delta}}{2\sigma^2}}
    = 2\exp\p{-\log \frac{2n}{\delta}}
    = \frac{\delta}{n}.
  \]
  A union bound completes the proof.
\end{proof}
\begin{lemma}
  If $X_1, \dots, X_K$ are independent zero-mean $\sigma$-subgaussian random variables, then
  \[
    \textstyle \E{\max_i X_i} \le \sigma \sqrt{2 \log K}.
  \]
  \label{lem:subgaussian_max}
\end{lemma}
\begin{proof}
Let $X = \max X_i$. Since each $X_i$ is $\sigma$-subgaussian, it follows that
  \[
    \E{e^{\lambda X_i}} \le \exp\p{\frac{\lambda^2 \sigma^2}{2}}.
  \]
  Using Jensen's inequality, we have
  \[
    \exp\p{\lambda\E{X}} \le \E{\exp\p{\lambda X}} = \E{\max \exp\p{\lambda
    X_i}} \le \sum_{i} \E{\exp\p{\lambda X_i}} \le K \exp\p{\frac{\lambda^2
    \sigma^2}{2}}.
  \]
  Rearranging, we have
  \[
    \E{X} \le \frac{\log K}{\lambda} + \frac{\lambda \sigma^2}{2}.
  \]
  Setting $\lambda = \frac{\sqrt{2 \log K}}{\sigma}$, we have
  $  \E{X} \le \sigma \sqrt{2 \log K}$ as needed
\end{proof}

\begin{lemma}
  If $\theta \sim \mc N(\pmt, \pvt)$ where $\pmt\in \R^d$ and $\pvt \in \R^{d \times d}$, then
$\E{\; \|\theta - \pmt\|_2\; } \le \sqrt{d \lambda_{\max}(\pvt)}$.
  \label{lem:gaus_norm}
\end{lemma}
\begin{proof}
  From~\cite{chandrasekaran2012convex}, the expected norm of a standard normal
  $d$-dimensional Gaussian is at most $\sqrt{d}$. Using the fact that
$\pvt^{-1/2} (\theta - \pmt) \sim \mc N(0, I)$,
  we have
  \[
    \E{\|\theta - \pmt\|_2} = \E{\|\pvt^{1/2} \pvt^{-1/2}(\theta - \pmt)\|_2}
    \le \|\pvt^{1/2}\|_2 \E{\|\pvt^{-1/2}(\theta - \pmt)\|_2} \le
    \sqrt{d\lambda_{\max}(\pvt)}
  \]
\end{proof}

\begin{lemma}[Lemma 2.2 in \citet{dasgupta2003elementary}]
  If $X \sim \chi^2(d)$, \ie $X = \sum_{i=1}^d X_i^2$, where $X_1 \LDOTS X_d$ are independent standard Normal random variables, then
  \begin{align*}
    \Pr\b{X \le \beta d} &\le (\beta e^{1-\beta})^{d/2} & \text{for any $\beta\in (0,1)$}, \\
    \Pr\b{X \ge \beta d} &\le (\beta e^{1-\beta})^{d/2} & \text{for any $\beta>1$}.
  \end{align*}
  \label{lem:chi_sq_conc}
\end{lemma}

\begin{lemma}[Hoeffding bound]
  If $\bar{X} = \frac{1}{n} \sum_{i=1}^n X_i$, where the $X_i$'s are independent
  $\sigma$-subgaussian random variables with zero mean, then
  \begin{align*}
    \max\left(\Pr\b{\bar{X} \ge t},\;\Pr\b{\bar{X} \le -t}\right)
     &\le \exp\p{-\frac{nt^2}{2\sigma^2}}
        &\text{for all $t>0$}, \\
  \max\left(
    \Pr\b{\overline{X} \le -\sigma
    \sqrt{\tfrac{2}{n}\log \tfrac{1}{\delta}}},\quad
    \Pr\b{\overline{X} \ge \sigma
    \sqrt{\tfrac{2}{n}\log \tfrac{1}{\delta}}}  \right) &\le \delta
    &\text{for all $\delta>0$}.
  \end{align*}
  \label{lem:hoeffding}
\end{lemma}

\subsection{KL-divergence}

We use some basic facts about KL-divergence. Let us recap the definition: given two distributions $P,Q$ on the same finite outcome space $\Omega$, KL-divergence from $P$ to $Q$ is
\[ \kl{P}{Q} := - \sum_{\omega \in \Omega} P(\omega) \log \tfrac{Q(\omega)}{P(\omega)} .\]

\begin{lemma}[High-probability Pinsker Inequality~\citep{T09}]
\label{lem:pinkser}
For any probability distributions $P$ and $Q$ over the same sample space and any arbitrary event $E$,
\[
P(E) + Q(\overline{E}) \ge \tfrac{1}{2}\, e^{-\kl{P}{Q}}.
\]
\end{lemma}

\begin{lemma}
\label{lem:bern_kl}
Let $P$ and $Q$ be Bernoulli distributions with means $p \in [1/2-\eps,
1/2+\eps]$ and $q \in [1/2-\eps, 1/2+\eps]$ respectively, with $\eps \le 1/4$. Then
  $\kl{P}{Q} \le \frac{7}{3}\,\eps^2$.
\end{lemma}
\begin{proof}
For any $\eps \le 1/4$,

\begin{align*}
\log \p{\frac{p(1-p)}{q(1-q)}} &\le
\log\p{\frac{1/4}{1/4-\eps^2}} \le
\log\p{\frac{1}{1-4\eps^2}}    \le
 \frac{14\, \eps^2}{3} \tag {By
    Lemma~\ref{lem:log_expansion}}
\\
  \kl{P}{Q} &= p \log \p{\frac{p}{q}} + (1-p) \log\p{\frac{1-p}{1-q}} \\
  &\le \p{\frac{1}{2} + \eps} \log \p{\frac{p(1-p)}{q(1-q)}}
    = \p{\frac{1}{2} + \eps} \frac{14 \eps^2}{3}
   \le \frac{7\eps^2}{2}.
\end{align*}
\end{proof}

\subsection{Linear Algebra}

We use several facts from linear algebra. In what follows, recall that
$\lambda_{\min}(M)$ and $\lambda_{\max}(M)$ denote the minimal and the maximal eigenvalues of matrix $M$, resp.
For two matrices $A,B$, let us write $B \succeq A$ to mean that $B-A$ is positive semidefinite.

\begin{lemma}
    $\lambda_{\max}(vv\tran) = \|v\|_2^2$ ~~ for any $v \in \R^d$.
  \label{lem:norm_eigen}
\end{lemma}
\begin{proof}
  $vv\tran$ has rank one, so it has one eigenvector with nonzero eigenvalue. $v$
  is an eigenvector since $(vv\tran)v = (v\tran v) v$, and it has eigenvalue
  $v\tran v = \|v\|_2^2$. This is the only nonzero eigenvalue, so
  $\lambda_{\max}(vv\tran) = \|v\|_2^2$.
\end{proof}

\begin{lemma} \label{lem:conj_succ}
  For symmetric matrices $A$, $B$ with $B$ invertible,
  \[
    B \succeq A \Longleftrightarrow I \succeq B^{-1/2} A B^{-1/2}
  \]
\end{lemma}
\begin{proof}
  \begin{align*}
    B \succeq A &\Longleftrightarrow x\tran B x \ge x\tran A x \tag{$\forall x$} \\
    &\Longleftrightarrow x\tran(B-A) x \ge 0 \tag{$\forall x$} \\
    &\Longleftrightarrow x\tran B^{1/2} (I-B^{-1/2}AB^{-1/2}) B^{1/2} x \ge 0 \tag{$\forall x$} \\
    &\Longleftrightarrow x\tran (I-B^{-1/2}AB^{-1/2}) x \ge 0 \tag{$\forall x$} \\
    &\Longleftrightarrow I \succeq B^{-1/2}AB^{-1/2}.
  \end{align*}
\end{proof}

\begin{lemma}
  If $A \succeq 0$ and $B \succeq 0$, then
$\lambda_{\min}(A + B) \ge \lambda_{\min}(A)$.
  \label{lem:min_ev_sum}
\end{lemma}
\begin{proof}
  \begin{align*}
    \lambda_{\min}(A + B) &= \min_{\|x\|_2 = 1} x\tran (A+B) x \\
    &= \min_{\|x\|_2 = 1} x\tran A x + x\tran B x \\
    &\ge \min_{\|x\|_2 = 1} x\tran A x \tag{because $x\tran B x \ge 0$} \\
    &= \lambda_{\min}(A)
  \end{align*}
\end{proof}

\subsection{Logarithms}

We use several variants of standard inequalities about logarithms.

\begin{lemma}
    $x \ge \log(ex)$ for all $x > 0$.
  \label{lem:xex}
\end{lemma}
\begin{proof}
  This is true if and only if $x - \log(ex) \ge 0$ for $x > 0$. To show this,
  observe that
  \begin{enumerate}
    \item At $x=1$, this holds with equality.
    \item At $x = 1$, the derivative is
      \[
        \frac{d}{dx} x - \log(ex) \bigg|_{x=1} = 1 - \frac{1}{x} \bigg|_{x=1} =
        0.
      \]
    \item The entire function is convex for $x > 0$, since
      \[
        \frac{d^2}{dx^2} x - \log(ex) = \frac{d}{dx} 1 - \frac{1}{x} =
        \frac{1}{x^2} > 0.
      \]
  \end{enumerate}
  This proves the lemma.
\end{proof}
\begin{corollary}
  $  x - \log x \ge \frac{e-1}{e} x$.
  \label{cor:e1ex}
\end{corollary}
\begin{proof}
  Using Lemma~\ref{lem:xex} and letting $z = x/e$,
  \[
    x - \log x = \frac{e-1}{e} x + \frac{1}{e}x - \log x = \frac{e-1}{e} x + z -
    \log(ez) \ge \frac{e-1}{e} x
  \]
\end{proof}
\begin{lemma}
$\log\p{\frac{1}{1-x}} \le \frac{7x}{6}$ for any $x \in [0,1/4]$.
  \label{lem:log_expansion}
\end{lemma}
\begin{proof}
  First, we note that
  \[
    \tfrac{d}{dx} \log\p{\tfrac{1}{1-x}} = 1-x (-(1-x)^{-2}) \cdot (-1)
    = \tfrac{1}{1-x}
    = \sum_{i=0}^\infty x^i.
  \]
  Integrating both sides, we have
  \[
    \log\p{\tfrac{1}{1-x}} = C + \sum_{i=0}^\infty \frac{x^i}{i},
  \]
  for some constant $C$ that does not depend on $x$. Taking $x = 0$ yields $C = 0$.
  Therefore,
  \[
    \log\p{\frac{1}{1-x}} \le x + \frac{x^2}{2} \sum_{i=0}^\infty x^i = x +
    \frac{x^2}{2(1-x)} = x\p{1 + \frac{x}{2(1-x)}} \le \frac{7x}{6}.
  \]
\end{proof}